\documentclass[10pt,twocolumn,letterpaper]{article}

\usepackage[pagenumbers]{cvpr} 

\usepackage[dvipsnames]{xcolor}

\usepackage{times}
\usepackage{epsfig}
\usepackage{graphicx}
\usepackage{amsmath}
\usepackage{amssymb}

\usepackage{multirow}
\usepackage{booktabs}

\usepackage{enumitem}

\usepackage{tcolorbox}

\usepackage{tikz}

\usepackage{amsfonts}
\usepackage{bm}
\usepackage{amsmath,amsthm,amssymb,rotating, mathrsfs}
\usepackage{footmisc}

\usepackage{bibunits}

\newtheorem{theorem}{Theorem}

\newtheorem{proposition}{Proposition}

\newtheorem{lemma}{Lemma}

\theoremstyle{definition}

\theoremstyle{remark}
\newtheorem{remark}{Remark}

\usepackage[linesnumbered,lined,boxed,commentsnumbered,ruled,vlined]{algorithm2e}

\SetCommentSty{mycommfont}

\usepackage{bm}

\usepackage{mathtools}
\usepackage{cancel} 

\usepackage{xparse}

\newcommand{\cI}{\mathcal{I}}
\newcommand{\cJ}{\mathcal{J}}

\newcommand{\cN}{\mathcal{N}}

\newcommand{\cQ}{\mathcal{Q}}

\DeclareMathAlphabet{\mathpzc}{OT1}{pzc}{m}{it}

\newcommand{\bbR}{\mathbb{R}}
\newcommand{\bbS}{\mathbb{S}}

\newcommand{\bI}{\bm{I}}

\newcommand{\bR}{\bm{R}}

\newcommand{\br}{\bm{r}}

\newcommand{\bt}{\bm{t}}

\newcommand{\bx}{\bm{x}}
\newcommand{\by}{\bm{y}}

\NewDocumentCommand{\norm}{mG{2}}{\big\|#1\big\|_{#2}}

\DeclareMathOperator{\tear}{\texttt{TEAR}}

\DeclareMathOperator{\SE}{SE}

\NewDocumentCommand{\seqp}{mG{n}}{{#1}_1-\cdots+ {#1}_{#2}}
\NewDocumentCommand{\seqm}{mG{n}}{{#1}_1-\cdots- {#1}_{#2}}

\newcommand{\myparagraph}[1]{\smallskip\noindent\textbf{#1.}}

\definecolor{cvprblue}{rgb}{0.21,0.49,0.74}
\usepackage[pagebackref,breaklinks,colorlinks,citecolor=cvprblue,linkcolor=cvprblue]{hyperref}
\usepackage{makecell}

\def\paperID{941} 
\def\confName{CVPR}
\def\confYear{2024}

\usepackage{soul}

\title{Scalable 3D Registration via Truncated Entry-wise Absolute Residuals}

\author{Tianyu Huang$^{1}$\thanks{\ Equal contribution} 
\and
Liangzu Peng$^{2}$\footnotemark[1]
\and
Ren\'e Vidal$^{2}$
\and
Yun-Hui Liu$^{1}$
\and
$^{1}$The Chinese University of Hong Kong
\and
$^{2}$University of Pennsylvania 
\and
$^{1}${\tt\small tyhuang, yhliu@mae.cuhk.edu.hk}
\and
$^{2}${\tt\small lpenn, vidalr@seas.upenn.edu}
}

\begin{document}
\maketitle

\begin{abstract}
    Given an input set of $3$D point pairs, the goal of outlier-robust $3$D registration is to compute some rotation and translation that align as many point pairs as possible. This is an important problem in computer vision, for which many highly accurate approaches have been recently proposed. Despite their impressive performance, these approaches lack scalability, often overflowing the 
    $16$GB of memory of a standard laptop to handle roughly $30,000$ point pairs. 
    In this paper, we propose a $3$D registration approach that can process more than ten million ($10^7$) point pairs with over $99\%$ random outliers. Moreover, our method is efficient, entails low memory costs, and maintains high accuracy at the same time. We call our method $\tear$~\footnote{\url{https://github.com/tyhuang98/TEAR-release}\label{footnote:code}}, as it involves minimizing an outlier-robust loss that computes Truncated Entry-wise Absolute Residuals. To minimize this loss, we decompose the original $6$-dimensional problem into two subproblems of dimensions $3$ and $2$, respectively, solved in succession 
    to global optimality via a customized branch-and-bound method. While branch-and-bound is often slow and unscalable, 
    this does not apply to $\tear$ as we propose novel bounding functions that are tight and computationally efficient. Experiments on various datasets are conducted to validate the scalability and efficiency of our method.
\end{abstract}

\section{Introduction}\label{sec:intro}

The \textit{$3$D registration} problem aims to find a rotation and translation that
best align an input set of 3D point pairs. Ideally, the alignment errors for all point pairs are small, and we call them \textit{inliers}. In practice, the inliers are contaminated by other point pairs, called \textit{outliers}, that induce significant alignment errors. Given a set of inlier and outlier point pairs, \textit{outlier-robust 3D registration} aims to align the 3D inlier point pairs via some rotation and translation.



In this paper, we tackle the 3D registration problem with extremely many outliers, via the proposed method that we call \textit{Truncated Entry-wise Absolute Residuals} ($\tear$). Numerically, $\tear$ can handle more than $10^7$ point pairs with $99.8\%$ random outliers, 
a setting in which no existing methods have been shown to succeed: They are either unscalable, inefficient, or inaccurate. In  \cref{subsection:prior-work}, we review prior 3D registration methods. In \cref{subsection:contribution}, we further highlight the $\tear$ approach and overview our contributions.

\subsection{Prior Art}\label{subsection:prior-work}

In this section, we briefly review several families of 3D robust registration methods that are relevant to our work.

First, observe that the outlier-robust 3D registration problem can be decoupled into two subproblems: (1) if the outliers are known, one can easily estimate a rotation and translation by SVD \cite{Arun-TPAMI1987,Horn-JOSAA1987,Horn-JOSAA1988}, and (2) if the true rotation and translation are given, one can easily remove the outliers. These observations turn immediately into an efficient alternating minimization method, which lies at the heart of classical approaches including \textit{iterative closest point} \cite{Besl-PAMI1992,Rusinkiewicz-2001}, \textit{iteratively reweighted least-squares} \cite{Coleman-TOMS1980,Ochs-SIAM-J-IS2015,Beck-SIAMOpt2015,Aftab-WCACV2015,Peng-NeurIPS2022,Peng-CVPR2023,Peng-arXiv2023b}, and \textit{graduated non-convexity} \cite{Blake-1987,Zhou-ECCV2016,Zach-ECCV2018,Yang-RA-L2020,Le-3DV2021,Sidhartha-CVPR2023}. Such an alternating minimization scheme is shown to be convergent in \cite{Aftab-WCACV2015,Peng-CVPR2023} under very general conditions, but it might not converge to a desired solution if the outlier ratio is high.

\begin{figure*}[!t]
\centering
\includegraphics[width=0.98\textwidth]{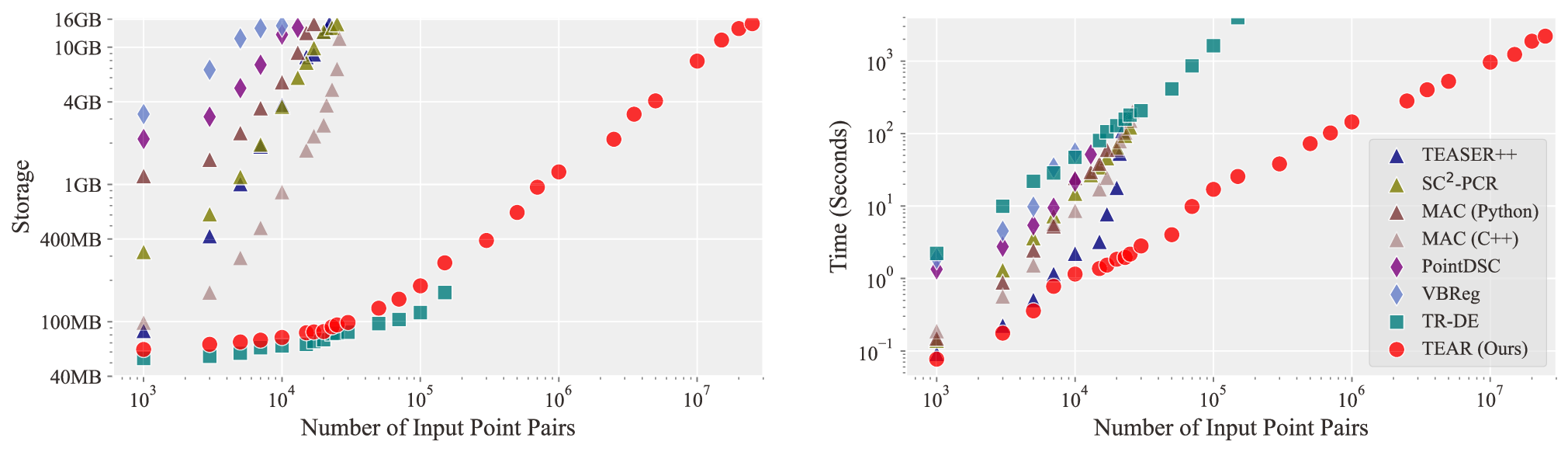}
    \\[-0.4em]
    \makebox[0.49\textwidth]{\footnotesize \quad \quad \ \ (a)}
    \makebox[0.49\textwidth]{\footnotesize \quad \quad \quad (b)}
    \\[-0.8em]
    \caption{Comparisons of $\tear$ (ours) to prior methods on random, synthetic, noisy data~(20 trials). The outlier ratio is set to $95\%$ and all presented methods find accurate solutions. \cref{fig:storage_time}a: $\tear$ is $10^3$ times more scalable than methods based on consistency graphs including TEASER++ \cite{Yang-T-R2021}, SC$^2$-PCR  \cite{Chen-CVPR2022b}, MAC \cite{Zhang-CVPR2023} and than deep learning methods including PointDSC \cite{Bai-CVPR2021} and VBReg \cite{Jiang-CVPR2023}. \cref{fig:storage_time}b: $\tear$ is $100$ times faster than TR-DE \cite{Chen-CVPR2022c}, a recent branch-and-bound method. }
    \label{fig:storage_time}
\end{figure*}

The RANSAC method \cite{Fischler-C-ACM1981} proceeds as follows: randomly sample a minimum number of point pairs (typically $3$ pairs), compute a rotation and translation that aligns them,  
measure the alignment error for all pairs, repeat these steps until a termination criterion is reached, and output the rotation and translation 
that give the smallest alignment error.
RANSAC could terminate with a correct estimation as soon as all pairs sampled in an iteration are inliers, but the probability of achieving so diminishes as the outlier ratio grows, suggesting that RANSAC, or its variants \cite{Torr-CVIU2000,Raguram-TPAMI2012,Barath-CVPR2018,Barath-CVPR2019}, can be inefficient in the presence of extremely many outliers.

Alternatively, one could formulate some non-convex objective (\eg, \textit{consensus maximization}) for outlier-robust registration, and solve it via the \textit{branch-and-bound} algorithm \cite{Daniel-Book2011}. Branch-and-bound guarantees global optimality by design, so it has served well in recent years as a validation tool \cite{Li-ICCV07,Olsson-TPAMI2008,Li-ICCV09,Yang-PAMI16,Bustos-PAMI16,Lian-PAMI17,Liu-ECCV18,Campbell-PAMI18,Cai-ISPRS-J-PRS2019,Li-ECCV2020,Chen-CVPR2022c,Liu-IJCV2022,Chen-TPAMI2022,Zhang-arXiv2023}. Its efficiency hinges on two critical aspects: the dimension of the search space and bounds on the objective. Care is needed to reduce the search space dimension or tighten the bounds, otherwise branch-and-bound could quickly be intractable as the problem scales up. The recent well-designed branch-and-bound method of \cite{Chen-CVPR2022c}, called TR-DE, takes more than $10^3$ seconds to handle $10^5$ point pairs (\cref{fig:storage_time}).

Instead of tackling the non-convex robust registration problem directly, one could consider relaxing it into a convex semidefinite program, which is typically solvable in polynomial time \cite{Briales-CVPR2017,Yang-ICCV2019,Iglesias-CVPR2020,Peng-ECCV2022,Yang-MP2023}. While semidefinite relaxations might recover the solution of the original non-convex objective, this recovery property could come at the cost of working with quadratically many optimization variables and constraints \cite{Yang-ICCV2019}, and thus at the expense of vast computation: State-of-the-art solvers need more than $7$ hours to solve such semidefinite programs and align 1,000 point pairs, even if the translation is given \cite[Table 3]{Yang-MP2023}.

The \textit{outlier removal} method \cite{Bustos-ICCV2015,Bustos-TPAMI2018} has a for loop: For each point pair, assume it is an inlier, then one can reason which point pair conflicts with this assumption. Doing so allows removing some point pairs and---this is a virtue---the point pairs to be removed are guaranteed to be outliers. While outlier removal often serves as a preprocessing step that facilitates subsequent alignment, it does so by charging a large amount of time (say, exponential in the variable dimension \cite{Bustos-TPAMI2018,Chin-CVPR2016} or at least quadratic in the number of points); \eg, on our laptop, the method of \cite{Bustos-TPAMI2018} takes more than 2 hours for $10^5$ point pairs with $95\%$ outliers.  

The next several methods we review rely on the so-called \textit{consistency graph}, a graph where a vertex denotes a point pair and an edge indicates two \textit{consistent} point pairs that can both be inliers. The consistency graph finds its early use in \cite{Leordeanu-ICCV2005,Enqvist-ICCV2009} and is the cornerstone of many recent methods for outlier-robust 3D registration \cite{Yang-T-R2021,Lusk-ICRA2021,Sun-RAL2022,Yan-TPAMI2022,Zhang-CVPR2023,Chen-TPAMI2023}. For example, using the consistency graph, the TEASER++ method computes a \textit{maximum clique}\footnote{A maximum clique of a graph is a complete subgraph containing the largest possible number of vertices, whereas a maximal clique is a complete subgraph not contained in any (other) maximum clique. \label{footnote:clique} } often containing most inliers and few outliers  \cite{Yang-T-R2021}; the SC$^2$-PCR method generalizes the consistency graph into a second-order version more discriminant between inliers and outliers \cite{Chen-CVPR2022b,Chen-TPAMI2023}; the MAC method generalizes the maximum clique formulation into computing \textit{maximal cliques}\footref{footnote:clique} \cite{Zhang-CVPR2023}. Even though these methods have established state-of-the-art performance, they have also brought an \textit{elephant} into the room: Computing a consistency graph uses memory quadratic in the number of point pairs, \eg, doing so for 30,000 point pairs would occupy the total 16GB memory of a standard laptop, which limits the applicability of all these methods to larger-scale robust 3D registration problems  (\cref{fig:storage_time}).

There have recently been many deep learning approaches developed to extract and match features of the input point clouds \cite{Wang-ICCV2019,Pais-CVPR2020,Choy-CVPR2020,Lee-ICCV2021,Guo-RAL2023,Yao-TPAMI2023,Yu-CVPR2023,Qin-CVPR2023}, the most relevant to ours are methods that perform robust 3D registration, such as PointDSC \cite{Bai-CVPR2021} and VBReg \cite{Jiang-CVPR2023}. PointDSC builds upon the consistency graph, and needs extra storage for a large network and (temporarily, during forward passes) the high-dimensional feature of each 3D point. VGReg builds upon PointDSC and uses a recurrent network that admits a variational interpretation, but it inherits the drawback of PointDSC of being not scalable  (\cref{fig:storage_time}).

\myparagraph{Summary} Since scalability has been \textit{a fly in the ointment} compromising the recent success of 3D registration methods, why not simply downsample huge-scale point clouds and perform registration from there? The answer to this sticking point is that downsampling ignores some input information that one could otherwise leverage, so it ultimately impinges upon performance; for numerical evidence, see, \eg, \cite[Table 2]{Jiang-CVPR2023}, \cite[Fig. 9]{Chen-TPAMI2023}, and \cref{fig:huge_down_rota}. This thus renders downsampling into a stopgap that eventually necessitates developing scalable and efficient registration methods.

\subsection{Our Contribution: TEAR}\label{subsection:contribution}
Which methodology, reviewed above, can be utilized to design a scalable approach for registrating ten million point pairs with extremely many outliers? Alternating minimization and RANSAC are known to be brittle at high outlier ratios. Semidefinite programs are costly to solve. Outlier removal typically needs quadratic time, constructing consistency graphs further consumes quadratic storage, and deep learning demands even more memory.

To design a scalable method, we advocate the branch-and-bound method. This might be surprising (if not doubtful): Conventional wisdom has it that branch-and-bound is slow and induces exponential running times! Contrary to the common wisdom, \cref{fig:storage_time} indicates the running time of our proposal ($\tear$) grows almost linearly. We achieve this by revising the problem-solving pipeline, from the problem formulation to mathematical derivations (of upper and lower bounds), and furthermore to implementation details. More explicitly, we make the following contributions:
\begin{itemize}
    \item (\textit{Problem Formulation}, \cref{subsection:problem-formulation,subsection:tear-off}) We formulate the 3D registration problem using the robust loss that we call TEAR, a shorthand for \textit{Truncated Entry-wise Absolute Residuals} (\cref{subsection:problem-formulation}). TEAR is similar in spirit to commonly seen robust losses (\eg, consensus maximization, truncated least-squares), but it has subtle differences that enable faster branch-and-bound algorithms to be derived. Moreover, in \cref{subsection:tear-off} we decompose TEAR into two subproblems of dimensions $3$ and $2$ respectively, which further facilitates developing branch-and-bound implementations. In fact, at a high level, our approach is very simple: We solve the two subproblems, one after another, by a basic branch-and-bound template.
    \item (\textit{Upper and Lower Bounds}, \cref{subsection:UB-LB,subsection:TEAR-CM-TLS}) The nontrivial part of our approach lies in deriving tight lower and upper bounds for the branch-and-bound method, and our key idea for achieving so is as follows (\cref{subsection:UB-LB}). To solve the $3$-dimensional subproblem, for example, our implementation searches a $2$-dimensional space (rather than $3$). In this implementation, we derive upper and lower bounds that can be computed via solving a specific $1$-dimensional problem in $O(N\log N)$ time, where $N$ is the total number of point pairs. We follow a similar route to solve the other $2$-dimensional subproblem to global optimality, and for simplicity, we call the final algorithm $\tear$. Via numerical comparisons (\cref{subsection:TEAR-CM-TLS}), we will show that using TEAR as the robust loss ensures the bounds are tighter than using the commonly used consensus maximization loss, and it also ensures the bounds are more efficient to compute than using the truncated least-squares loss. 
    \item (\textit{Experiments}, \cref{section:standard-experiments,section:huge-scale-experiments}) In \cref{section:standard-experiments} we perform standard experiments on synthetic and real data, showing that   $\tear$  reaches state-of-the-art accuracy while being more efficient in most cases. In \cref{section:huge-scale-experiments} we perform experiments on large-scale point clouds, presenting $\tear$ as a unique method that can handle ten million ($10^7$) point pairs in the presence of extremely many random outliers ($99.8\%$).
\end{itemize}

\section{The Design of TEAR}
This section introduces the design of $\tear$. In \cref{subsection:problem-formulation}, we revisit commonly used formulations for 3D registration and their drawbacks, thus motivating our proposal of a novel formulation called \textit{Truncated Entry-wise Absolute Residuals} (TEAR). In \cref{subsection:tear-off}, we decompose TEAR into easier subproblems. In \cref{subsection:UB-LB}, we describe how to solve the subproblems using branch-and-bound. In \cref{subsection:TEAR-CM-TLS}, we provide numerical validation that TEAR overcomes the drawbacks of other formulations and can be solved more efficiently.

\subsection{Problem Formulation: TEAR}\label{subsection:problem-formulation}
\myparagraph{Rethink Existing Formulations} Recall that our goal is to find some  3D rotation $\bR^*$ and translation $\bt^*$ that best aligns an input set of 3D point pairs $\{ (\by_i,\bx_i )\}_{i=1}^N$ containing a large fraction of (random) outliers. The first step of designing $\tear$ is to choose, if not to propose, an outlier-robust problem formulation that admits scalable algorithms. To this end, we recall two highly robust losses often used in geometric vision, namely \textit{Consensus Maximization} (CM) and \textit{Truncated Least-Squares} (TLS):
\begin{align}
    \max_{\substack{(\bR,\bt)\in \SE(3)}} &\sum_{i=1}^N \bm{1}\left(\|\by_i - \bR \bx_i - \bt \|_2 \leq \xi_i \right),  \label{eq:CM} \tag{\textcolor{cvprblue}{CM}} \\ 
    \min_{\substack{(\bR,\bt)\in \SE(3)}} & \sum_{i=1}^N \min \left\{\|\by_i - \bR \bx_i - \bt \|_2^2,\ \xi_i^2 \right\}.  \label{eq:TLS} \tag{\textcolor{cvprblue}{TLS}}
\end{align}
Here, $\bm{1}(\cdot)$ denotes the indicator function, and $\xi_i\geq 0$ is a threshold hyper-parameter, such that, if the residual $\|\by_i - \bR \bx_i - \bt \|_2$ is larger than $\xi_i$, then $(\by_i,\bx_i)$ is regarded as an outlier (with respect to $\bR,\ \bt$). Therefore, \ref{eq:CM} aims to minimize the number of outliers, whereas \ref{eq:TLS} furthermore minimizes the residuals of inliers using least-squares.

In the context of branch-and-bound, \ref{eq:CM} has been popular, as its upper and lower bounds are relatively easy to derive. However, using branch-and-bound for \ref{eq:CM} has a subtle yet crucial drawback. Indeed, note that in many cases two different rotations and translations could correspond to the same number of outliers (as determined by $\xi_i$), but the corresponding residuals are hardly the same. Note then that \ref{eq:CM} only counts the number of outliers with respect to the current rotation and translation, and it never calculates the sum of residuals. These imply the upper and lower bounds of \ref{eq:CM} are usually loose, which would jeopardize, consequentially, the efficiency of branch-and-bound.

Is \ref{eq:TLS} suitable for branch-and-bound? Here are some concerns, though. First, to our knowledge, no prior work applied branch-and-bound to \ref{eq:TLS}: The technical challenge is deriving the corresponding upper and lower bounds; the conceptual challenge is the impression that branch-and-bound would be slow anyway. Our second, major, concern has a deeper cause, which we will illustrate later in \cref{subsection:TEAR-CM-TLS}.

\myparagraph{TEAR} In light of the above concerns, we propose \textit{Truncated Entry-wise Absolute Residuals} (TEAR):
\begin{equation}\label{eq:TEAR}
    \min_{\substack{(\bR,\bt)\in \SE(3)}} \sum_{i=1}^N \min\left\{ \|\by_i - \bR \bx_i - \bt \|_1,\ \xi_i \right\}. \tag{\textcolor{red}{TEAR}} 
\end{equation}
Different from \ref{eq:CM}, \ref{eq:TEAR} evaluates the sum of (truncated) residuals. Different from \ref{eq:TLS}, \ref{eq:TEAR} evaluates the sum of the entry-wise absolute values of $\by_i - \bR \bx_i - \bt$ ($\ell_1$ loss), rather than the squared sum ($\ell_2$ loss). The benefit of having these differences is computational: With \ref{eq:TEAR} we can derive a branch-and-bound algorithm with tight and efficiently computable upper and lower bounds (\cref{subsection:UB-LB,subsection:TEAR-CM-TLS}).
\begin{remark}[Truncated Least Unsquared Deviations]
    In a recent preprint \cite{Lee-arXiv2023}, the following robust loss was considered (translated to the context of robust 3D registration):
    \begin{equation}
        \min_{\substack{(\bR,\bt)\in \SE(3)}} \sum_{i=1}^N \min \left\{\|\by_i - \bR \bx_i - \bt \|_2,\ \xi_i \right\}.  \label{eq:TLUD} \tag{\textcolor{cvprblue}{TLUD}}
    \end{equation}
    \ref{eq:TLUD} uses the unsquared $\ell_2$ norm rather than the $\ell_1$ norm of \ref{eq:TEAR}, and if $\by_i - \bR \bx_i - \bt$ were a scalar, \ref{eq:TLUD} would be equivalent to \ref{eq:TEAR}. One potential disadvantage of \ref{eq:TLUD} is that $\|\by_i - \bR \bx_i - \bt \|_2$ is not separable and this might bring computational difficulties (\eg, see \cite[Table 1]{Lee-arXiv2023}). In \cref{subsection:tear-off,subsection:UB-LB}, we will simplify \ref{eq:TEAR}  and improve computational efficiency by leveraging its separable residual.
\end{remark}


\subsection{Tear Off: Decomposition of TEAR}\label{subsection:tear-off}
Branching over the $6$-dimensional space $\SE(3)$ would be inefficient (as prior work showed), and directly applying branch-and-bound to \ref{eq:TEAR} would lead to unscalable implementations. Therefore, in order to implement a scalable branch-and-bound method, in this section we decompose \ref{eq:TEAR} into lower-dimensional problems (tears).

Denote by $y_{ij}$, $\br_j^\top$, and $t_j$ the $j$-th row of $\by_i$, $\bR$, and $\bt$, respectively. The residual $\|\by_i - \bR \bx_i - \bt \|_1$ has 3 summands:
\begin{equation*}
    \|\by_i - \bR \bx_i - \bt \|_1=\sum_{j=1}^3 | y_{ij} - \br_j^\top \bx_i - t_j|.
\end{equation*}
While $\|\by_i - \bR \bx_i - \bt \|_1$ has $6$ degrees of freedom ($\bR$ and $\bt$), each summand $| y_{ij} - \br_j^\top \bx_i - t_j|$ has only $3$ degrees of freedom ($\br_j\in \bbS^2:=\{\br\in\bbR^3: \| \br\|=1\}$ and $t_j\in \bbR$). This motivates us to approximate the \ref{eq:TEAR} problem as follows. First, with some threshold hyper-parameter\footnote{The choices of hyper-parameters are discussed in the appendix. \label{footnote:hyperparameters} } $\xi_{i1}$, we \textit{tear off} the first summand from \ref{eq:TEAR}, targeting at 
\begin{equation}\label{eq:tear-1}
    \min_{\substack{\br_1\in \bbS^2,t_1\in \bbR } } \sum_{i=1}^N \min\left\{| y_{i1} - \br_1^\top \bx_i - t_1|,\ \xi_{i1} \right\}. \tag{\textcolor{red}{TEAR-1}}
\end{equation}
Solving \ref{eq:tear-1} gives a solution $(\hat{\br}_1, \hat{t}_1 )$ revealing the set 
\begin{equation}\label{eq:I1-after-tear1}
    \hat{\cI}_1=\left\{ i: | y_{i1} - \hat{\br}_1^\top \bx_i - \hat{t}_1| \leq \xi_{i1} \right\},
\end{equation}
which contains the indices of potential inliers. Then, we tear off the second summand $| y_{i2} - \br_2^\top \bx_i - t_2|$, focusing on 
\begin{equation}\label{eq:tear-2}
    \begin{split}
        \min_{\substack{\br_2\in \bbS^2,t_2\in \bbR } } & \sum_{i\in \hat{\cI}_1 } \min\left\{| y_{i2} - \br_2^\top \bx_i - t_2|,\ \xi_{i2} \right\} \\ 
        \textnormal{s.t.} &\quad \quad  \br_2^\top \hat{\br}_1 = 0 
    \end{split} \tag{\textcolor{red}{TEAR-2}}
\end{equation}
where $\xi_{i2}$ is another threshold hyper-parameter\footref{footnote:hyperparameters}. Note that the extra constraint $\br_2^\top \hat{\br}_1 = 0$ in \ref{eq:tear-2} ensures the resulting rotation has orthogonal rows. Moreover, it implies \ref{eq:tear-2} has $2$ degrees of freedom, easier to solve than \ref{eq:tear-1} after a suitable reparameterization.

\begin{figure}[!t]
    \centering
    \includegraphics[width=0.4\textwidth]{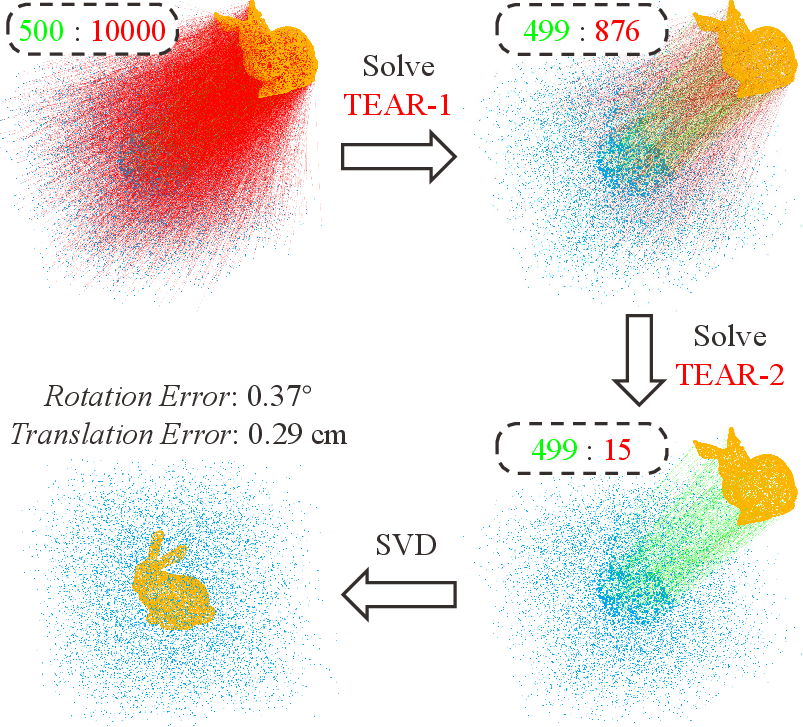}
   \\[-0.8em]
    \caption{The pipeline of $\tear$ visualized (\cf \cref{subsection:tear-off}). Green (\textit{resp.} red)  values denote the numbers of inliers (\textit{resp.} outliers), and green (\textit{resp.} red) lines denote inlier (\textit{resp.} outlier) point pairs. Top left: Input point pairs; top right: point pairs indexed by $\hat{\cI}_1$ \eqref{eq:I1-after-tear1} after solving \ref{eq:tear-1}; bottom right: point pairs indexed by $\hat{\cI}_2$ \eqref{eq:I2-after-tear2} after solving \ref{eq:tear-2}; bottom left: the final output. }
    \label{fig:TEAR-pipeline}
\end{figure}

Shall we proceed similarly for the third summand $| y_{i3} - \br_3^\top \bx_i - t_3|$? In fact, no more tear is needed: Given some optimal $(\hat{\br}_1, \hat{t}_1)$ and $(\hat{\br}_2, \hat{t}_2)$ from \ref{eq:tear-1} and \ref{eq:tear-2} respectively, one needs to set $\br_3$ to be $\pm(\hat{\br}_1\times \hat{\br}_2)$ to satisfy the rotation constraint (the sign is chosen to ensure the determinant is $1$), after which one needs only to find the final translation component. But what we do to implement $\tear$ is even simpler: Solving \ref{eq:tear-2} yields again a set
\begin{equation}\label{eq:I2-after-tear2}
    \hat{\cI}_2:=\left\{ i\in \hat{\cI}_1: | y_{i2} - \hat{\br}_2^\top \bx_i - \hat{t}_2| \leq \xi_{i2} \right\},
\end{equation}
which is expected to contain very few outlier indices (\eg, see \cref{fig:TEAR-pipeline}), and next, we apply SVD \cite{Arun-TPAMI1987,Horn-JOSAA1987,Horn-JOSAA1988} to the remaining point pairs indexed by $\hat{\cI}_2$ and this gives a rotation and translation estimate as the final output.

In summary, $\tear$ consists of solving \ref{eq:tear-1} and then \ref{eq:tear-2}, followed by an SVD; this is illustrated in \cref{fig:TEAR-pipeline} where we use the Stanford Bunny point cloud \cite{Curless-CCGIT1996} as an example. We solve \ref{eq:tear-1} and \ref{eq:tear-2} by a tailored branch-and-bound method, and we discuss that in \cref{subsection:UB-LB}.

\subsection{Branch and Bound with Tears}\label{subsection:UB-LB}
As mentioned, \ref{eq:tear-2} is easier to solve, so here we only describe how to solve \ref{eq:tear-1} using branch-and-bound, a global optimization technique that we assume, for conciseness, the reader is familiar with. The full recipe for solving \ref{eq:tear-1} is in the appendix, and the code is available\footref{footnote:code}.

While \ref{eq:tear-1} has $3$ degrees of freedom, it would suffice to branch over the sphere $\bbS^{2}$ where $\br_1$ resides or equivalently over the rectangle $\cQ:= [0,2\pi)\times [0,\pi]$, as each point $(\alpha,\beta)$ in this rectangle corresponds uniquely to a unit vector $\br_1=[\sin\beta \cos\alpha, \sin\beta \sin \alpha, \cos\beta ]$. To implement branch-and-bound, we need to consider two key steps:
\begin{itemize}
    \item (\textit{Upper Bound}) Given a point in $\cQ$, that is given a unit vector $\br_1$, we minimize \ref{eq:tear-1} in variable $t_1$, \ie, solve
    \begin{equation}\label{eq:tear-1-UB}
    \min_{\substack{t_1\in \bbR } } \sum_{i=1}^N \min\left\{| y_{i1} - \br_1^\top \bx_i - t_1|,\ \xi_{i1} \right\}. 
\end{equation}
    By construction, the minimum value of \eqref{eq:tear-1-UB} is always an upper bound of the minimum of \ref{eq:tear-1}. 
    \item (\textit{Lower Bound}) Given a subrectangle $[\alpha_1,\alpha_2]\times [\beta_1,\beta_2]$ in $\cQ$, consider the following program:
    \begin{equation}\label{eq:tear-1-LB}
        \begin{split}
            \min_{\substack{\alpha\in [\alpha_1,\alpha_2] \\ \beta\in [\beta_1,\beta_2] \\ t_1\in \bbR } } & \sum_{i=1}^N \min\left\{| y_{i1} - \br_1^\top \bx_i - t_1|,\ \xi_{i1} \right\} \\ 
        \textnormal{s.t.} &\quad  \br_1=[\sin\beta \cos\alpha, \sin\beta \sin \alpha, \cos\beta ]^{\top}
        \end{split}
    \end{equation}
    Note that \eqref{eq:tear-1-LB} is identical to \ref{eq:tear-1} except that $\br_1$ is now constrained such that the two angles $\alpha$ and $\beta$ are bounded in the intervals $[\alpha_1,\alpha_2]$ and $[\beta_1,\beta_2]$. While \eqref{eq:tear-1-LB} is a subproblem encountered during branch-and-bound, solving \eqref{eq:tear-1-LB} efficiently is not easy. The key step that we take is to relax \eqref{eq:tear-1-LB} a little bit more so that a lower bound on the minimum of \eqref{eq:tear-1-LB} can be efficiently calculated (see the appendix for details). By construction, this lower bound is also a lower bound of \ref{eq:tear-1} subject to $\alpha\in [\alpha_1,\alpha_2]$ and $\beta\in [\beta_1,\beta_2]$. This is what we meant by computing a lower bound for the branch-and-bound method.   
\end{itemize}
We use the following statement to encapsulate the details of computing the desired upper and lower bounds:
\begin{theorem}\label{theorem:UB-LB}
    We can solve \eqref{eq:tear-1-UB} in $O(N\log N)$ time and compute a ``tight'' lower bound of \eqref{eq:tear-1-LB} also in $O(N\log N)$ time.
\end{theorem}
Behind Theorem \ref{theorem:UB-LB} are two novel, non-trivial algorithms that we propose to compute the bounds~(see the appendix for algorithmic details), and they are the game changers that enable \ref{eq:tear-1} to be solved highly efficiently. The counterparts of Theorem \ref{theorem:UB-LB} for \ref{eq:CM} and \ref{eq:TLS} are also derived, but they lead to less efficient branch-and-bound solvers than what we proposed for \ref{eq:tear-1}. We consolidate this claim with theoretical and numerical insights in \cref{subsection:TEAR-CM-TLS}.

\subsection{TEAR Versus CM and TLS}\label{subsection:TEAR-CM-TLS}
Previously, in \cref{subsection:UB-LB}, we advocated performing branch-and-bound with tears. We now validate this design choice. 
Specifically, we follow exactly the same logic as in \cref{subsection:tear-off,subsection:UB-LB} to derive branch-and-bound methods for \ref{eq:CM} and \ref{eq:TLS}, and then compare them to \ref{eq:TEAR}. Again, at the heart of the derivation is to compute the upper and lower bounds (\cf, \cref{subsection:UB-LB}). Some details are given below.

\begin{figure}[!t]
    \centering
    \includegraphics[width=0.478\textwidth]{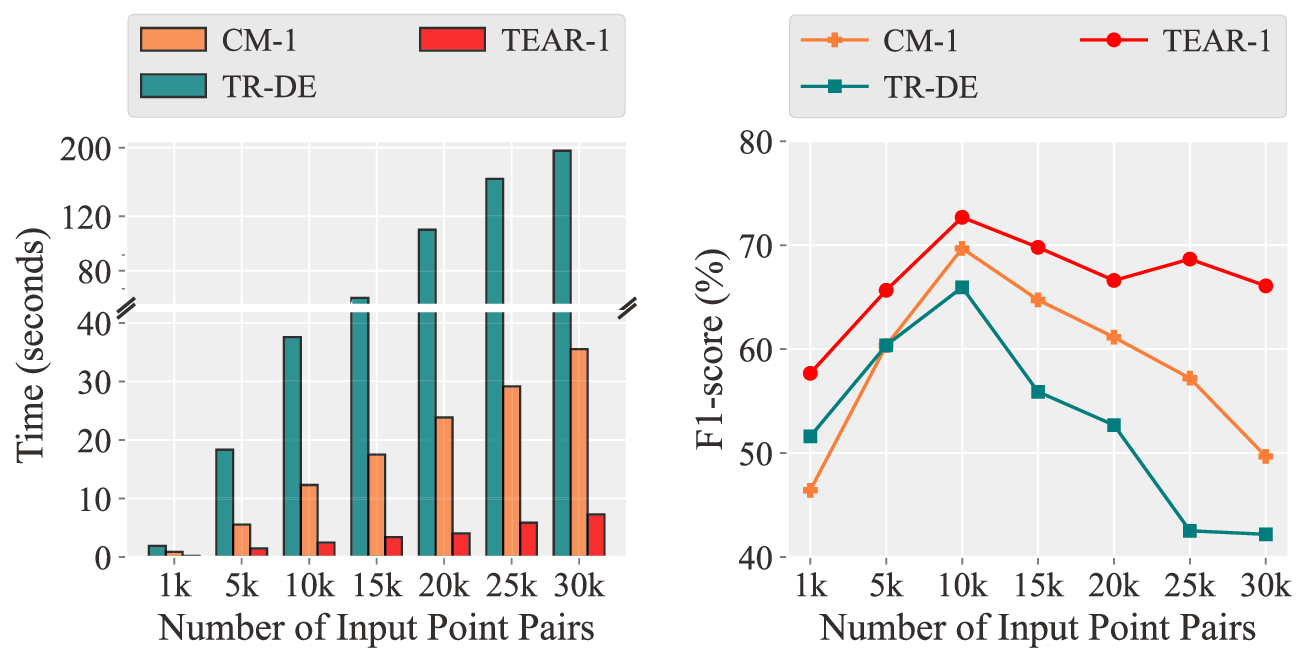}
    \\[-0.7em]
    \makebox[0.235\textwidth]{\footnotesize \quad \quad \ (a)}
    \makebox[0.235\textwidth]{\footnotesize \quad \quad \quad (b)}
    \\[0.2em]
    \includegraphics[width=0.478\textwidth]{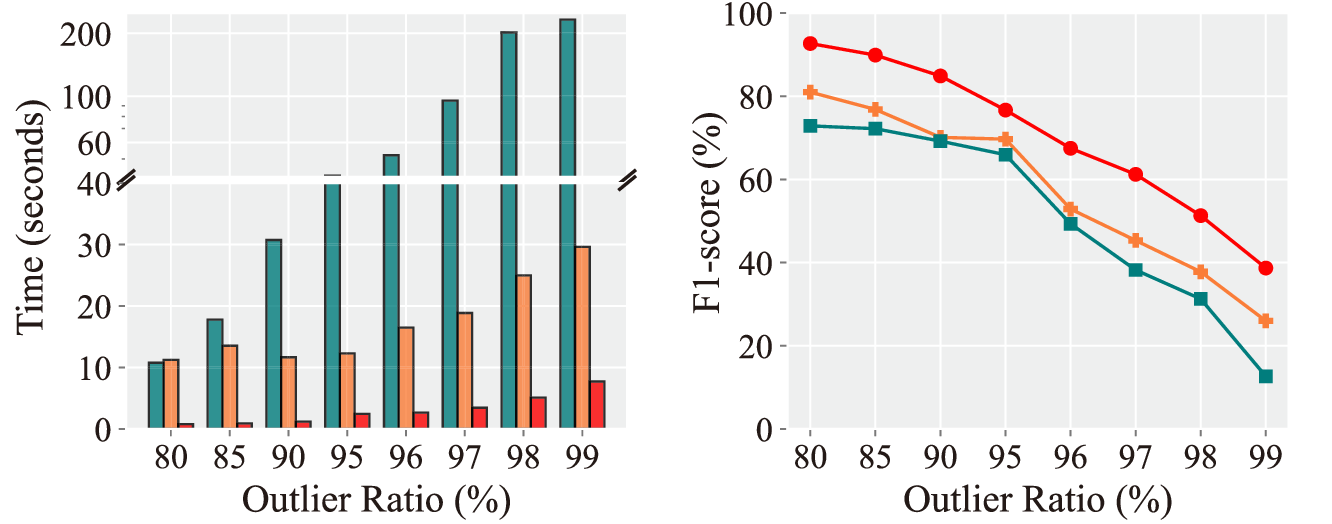}
    \\[-0.5em]
    \makebox[0.235\textwidth]{\footnotesize \quad \quad \ (c)}
    \makebox[0.235\textwidth]{\footnotesize \quad \quad \quad (d)}
    \\[-0.8em]
    \caption{Solve \ref{eq:tear-1} and \ref{eq:CM-1} via branch-and-bound on random, synthetic, noisy data (\cref{subsection:TEAR-CM-TLS}). TR-DE \cite{Chen-CVPR2022c}, a recent branch-and-bound method, is also compared. Outlier ratio: 95$\%$ (\cref{fig:TEAR-CM}a, \cref{fig:TEAR-CM}b); $N=10000$ (\cref{fig:TEAR-CM}c, \cref{fig:TEAR-CM}d). 30 trials.}
    \label{fig:TEAR-CM}
\end{figure}

\myparagraph{TEAR Versus CM} Recall that TR-DE \cite{Chen-CVPR2022c} is a branch-and-bound method consisting of searching a $3$-dimensional space to optimize some consensus maximization objective comparable to \ref{eq:tear-1}, with upper and lower bounds computed in $O(N)$ time. However, similarly in \cref{subsection:UB-LB}, one could optimize the same objective by searching instead in a $2$-dimensional space, with upper and lower bounds computable in $O(N\log N)$ time; this is in fact the core idea of the recent works \cite{Zhang-arXiv2023,Li-arXiv2023} to accelerate branch-and-bound. We contextualize this idea by implementing our version to solve a consensus maximization counterpart of \ref{eq:tear-1}:
\begin{equation}\label{eq:CM-1}
    \max_{\substack{\br_1\in \bbS^2,t_1\in \bbR } } \sum_{i=1}^N \bm{1} \left( | y_{i1} - \br_1^\top \bx_i - t_1| \leq \xi_{i1}  \right). 
    \tag{CM-1}
\end{equation}
\cref{fig:TEAR-CM} visualizes the differences between these methods: TR-DE is $5$ times slower than \ref{eq:CM-1} as \ref{eq:CM-1} searches a space of $1$ dimension lower; \ref{eq:CM-1} is more than $5$ times slower than \ref{eq:tear-1} as \ref{eq:tear-1} better leverages numerical values of the residuals than the binary truncation of \ref{eq:CM-1}. Finally, \cref{fig:TEAR-CM} shows \ref{eq:tear-1} has higher F1-scores, meaning that \ref{eq:tear-1} is more robust to outliers than \ref{eq:CM-1}.
\begin{remark}
    Since TR-DE \cite{Chen-CVPR2022c} was implemented with a single thread, experiments in \cref{subsection:TEAR-CM-TLS} (\cref{fig:TEAR-CM,fig:TEAR-TLS}) all use a single thread for fair comparison. In experiments of all other sections, we use a parallel implementation of $\tear$ that is multiple times faster than the single thread version.
\end{remark}

\myparagraph{TEAR Versus TLS} Similarly, we solve via branch-and-bound the TLS counterpart of \ref{eq:tear-1}:
\begin{equation}\label{eq:TLS-1}
    \min_{\substack{\br_1\in \bbS^2,t_1\in \bbR } } \sum_{i=1}^N \min\left\{ (y_{i1} - \br_1^\top \bx_i - t_1)^2,\ \xi_{i1}^2 \right\}.
    \tag{TLS-1}
\end{equation}
To our knowledge, branch-and-bound algorithms have not been applied to \ref{eq:TLS-1}, so we derive lower and upper bounds for such implementation and specify the computation complexities below (compare this to \cref{theorem:UB-LB}):
\begin{proposition}\label{prop:TLS-1-UB-LB}
    At each iteration of the branch-and-bound method for solving \ref{eq:TLS-1}, we can compute an upper bound in $O(N\log N)$ time via solving 
    \begin{equation}\label{eq:TLS-1-UB}
        \min_{\substack{t_1\in \bbR } } \sum_{i=1}^N \min\left\{ (y_{i1} - \br_1^\top \bx_i - t_1)^2,\ \xi_{i1}^2 \right\},
    \end{equation}
    and we can compute a lower bound in $O(N^2)$ time. 
\end{proposition}
\begin{remark}
    \cite[Algorithm 2]{Yang-T-R2021} is not necessarily optimal to \eqref{eq:TLS-1-UB} as it only finds the least-squares solution at a maximum consensus set, while an $O(N\log N)$ optimal algorithm was described in \cite{Liu-JCGS2019}. Computing lower bounds entails solving a harder problem than \eqref{eq:TLS-1-UB} (\eg, compare \eqref{eq:tear-1-UB} and \eqref{eq:tear-1-LB}), and for the moment we are only able to do so in $O(N^2)$ time.
\end{remark}
\cref{fig:TEAR-TLS} compares the efficiency of solving \ref{eq:tear-1} and \ref{eq:TLS-1}. \cref{fig:TEAR-TLS}a shows \ref{eq:TLS-1} takes slightly fewer iterations than \ref{eq:tear-1}, meaning that the bounds for \ref{eq:TLS-1} are slightly tighter. On the other hand, since computing bounds for \ref{eq:TLS-1} is more time-consuming (\cf Theorem \ref{theorem:UB-LB}, Proposition \ref{prop:TLS-1-UB-LB}), the overall running time of \ref{eq:TLS-1} grows more rapidly than \ref{eq:tear-1} as the number of point pairs increases. Finally, \cref{fig:TEAR-TLS}b shows \ref{eq:TLS-1} and \ref{eq:tear-1} have similar F1-scores, suggesting they are comparable in rejecting outliers.

\begin{figure}[!t]
    \centering
    \includegraphics[width=0.478\textwidth]{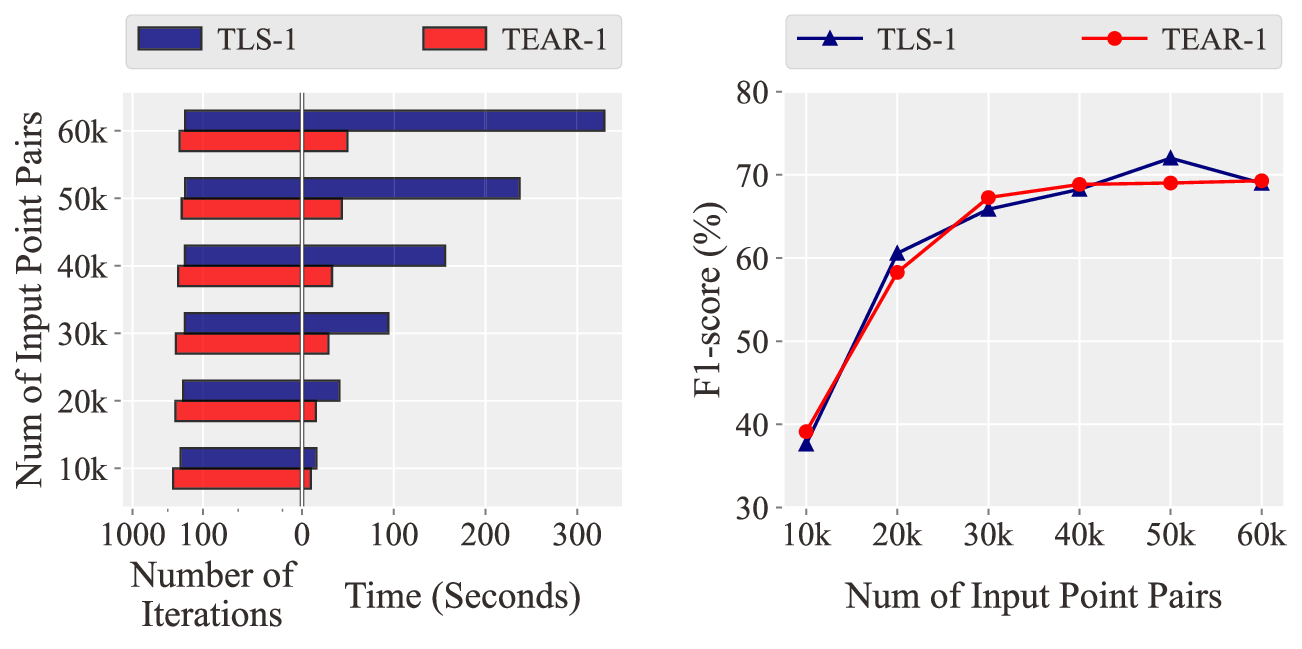}
    \\[-0.8em]
    \makebox[0.235\textwidth]{\footnotesize \quad \quad \ (a)}
    \makebox[0.235\textwidth]{\footnotesize \quad \quad \quad (b)}
    \\[-0.7em]
    \caption{Solve \ref{eq:tear-1} and \ref{eq:TLS-1} via branch-and-bound on random, synthetic, noisy data (\cref{subsection:TEAR-CM-TLS}). Outlier ratio: 99$\%$. 30 trials.}
    \label{fig:TEAR-TLS}
\end{figure}

\section{Standard Experiments}\label{section:standard-experiments}
Here we conduct experimental comparisons that are standard as in prior works. These include experiments on synthetic data (\cref{sec:syn_exper}) and real datasets (\cref{sec:real_exper}).

\myparagraph{Setup} For a fair comparison, we evaluate all methods on a laptop equipped with an Intel Core i7-10875H@2.3 GHz and 16GB of RAM. For methods implemented in PyTorch, we slightly adjust their codes so that they run on CPUs; note that this typically only affects running times, not accuracy. We run all methods in Python or through the provided Python interfaces of the C++ codes. Since the Python and C++ codes of MAC \cite{Zhang-CVPR2023} behave differently, we compare both versions, namely MAC (Python) and MAC (C++). The maximum number of iterations for RANSAC is set to 100k.

\myparagraph{Evaluation Metrics} Following \cite{Bai-CVPR2021, Chen-CVPR2022b, Chen-CVPR2022c}, we use $5$ metrics for evaluation: Registration Recall~(RR), F1-score (F1), Rotation Error (RE), Translation Error (TE), and Time. RR denotes the percentage of successful registration where the RE and TE are below specific thresholds, \eg, ($15^{\circ},\ 30$cm) for 3DMatch dataset, ($5^{\circ},\ 60$cm) for KITTI dataset, and ($3^{\circ},\ 50$cm) for ETH dataset. F1 denotes the harmonic mean of precision and recall~\cite{Bai-CVPR2021}. Finally, we report the peak memory consumption of an algorithm during its execution.




\begin{figure}[!t]
    \centering
    \includegraphics[width=0.478\textwidth]{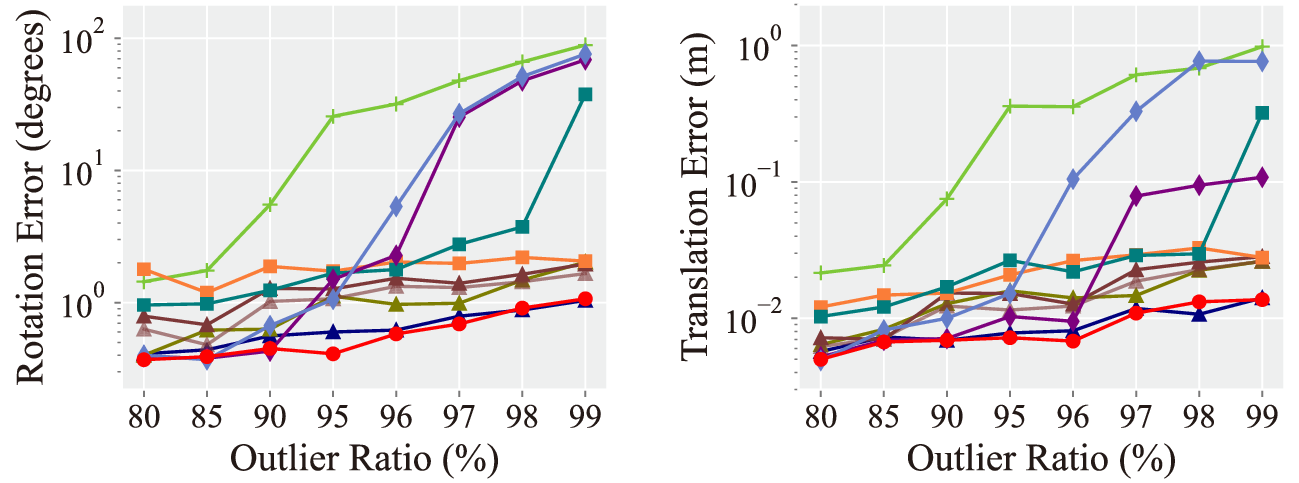}
    \\[-0.5em]
    \makebox[0.235\textwidth]{\footnotesize \quad \quad \ (a)}
    \makebox[0.235\textwidth]{\footnotesize \quad \quad \quad (b)}
    \\[0.4em]
    \includegraphics[width=0.478\textwidth]{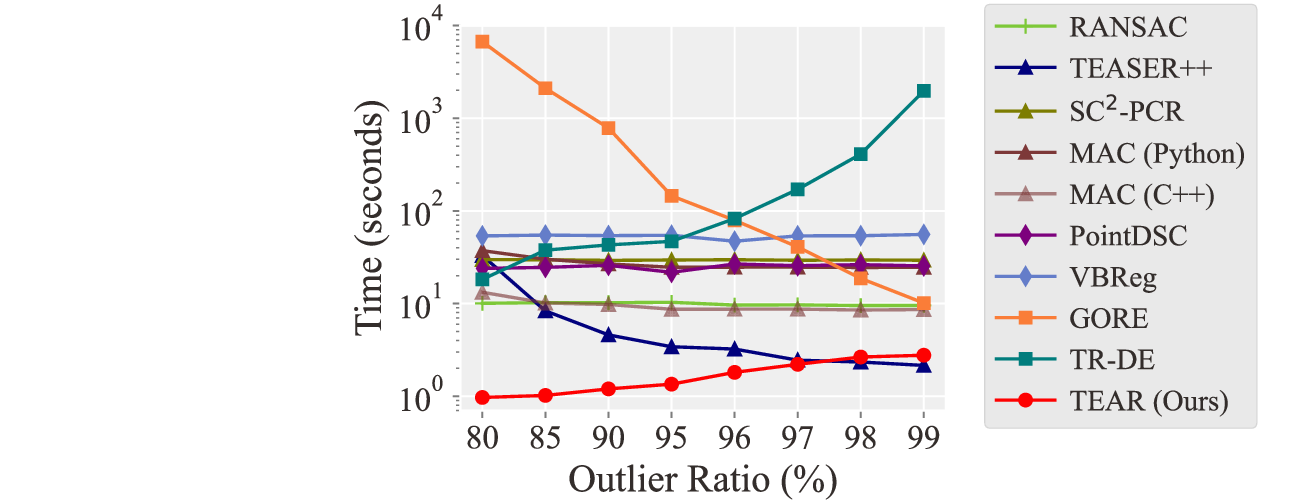}
    \\[-0.4em]
    \makebox[0.24\textwidth]{\footnotesize \quad \quad \ \ (c)}
    \\[-0.6em]
    \caption{Synthetic Experiments (\cref{sec:syn_exper}).  $N=10000$, 30 trials. }
    \label{fig:robust_outlier}
\end{figure}

\subsection{Experiments on Synthetic Data}\label{sec:syn_exper}
\myparagraph{Data Generation} Following~\cite{Bustos-TPAMI2018, Chen-CVPR2022b}, we generate synthetic point pairs as follows. First, we randomly sample $N$ points $\{\bx_i\}_{i=1}^N$ from $\mathcal{N}(0, \bI_3)$. Then, we apply a random rotation $\bR^*$ and translation $\bt^*$ to each $\bx_i$ and add some Gaussian noise $\bm{\epsilon}_i \sim \cN(0,\sigma^2\bI_3)$ of variance $\sigma=0.01$; that is $\by_i=\bR^*\bx_i + \bt^* + \bm{\epsilon}_i$. This gives $N$ (noisy) inlier pairs $\{\bx_i,\by_i\}_i^N$. Next, to generate outliers, we replace a fraction of $\by_i$'s with random Gaussian points sampled from $\mathcal{N}(0, \tau^2\bI_3)$, where $\tau$ is set to $1.67$. Almost all methods require an inlier threshold $\xi_i$, and we set it to $5.54\sigma$.

\myparagraph{Results} $\tear$ is shown to be scalable~(\cref{fig:storage_time}), accurate (\cref{fig:robust_outlier}a, \cref{fig:robust_outlier}b) and more efficient in most cases (\cref{fig:robust_outlier}c).

\subsection{Experiments on Real Data}
\label{sec:real_exper}
Here we report experimental results on three real-world datasets, namely 3DMatch (indoor scenes) \cite{Zeng-CVPR2017}, KITTI (outdoor scenes) \cite{Geiger-IJRR2013}, and ETH (outdoor, large-scale) \cite{Theiler-ISPRS2014}. 


\myparagraph{3DMatch} We follow \cite{Yang-T-R2021, Chen-CVPR2022b} to use 3DSmoothNet~\cite{Gojcic-CVPR2019} as the feature descriptor for the 3DMatch dataset~\cite{Zeng-CVPR2017}. The results are shown in \cref{tab:3DMatch_comparison}. Note that MAC (Python) runs out of the $16$GB memory, so we compared its C++ version.


\myparagraph{KITTI} Following \cite{Bai-CVPR2021, Zhang-CVPR2023}, we use FPFH \cite{Rusu-ICRA2009} as the feature descriptor of the KITTI dataset \cite{Geiger-IJRR2013}. In  \cref{tab:kitti_comparison}, MAC (Python) works well, and MAC (C++) runs out of memory. 

\myparagraph{ETH} Following \cite{Theiler-ISPRS2014,Li-arXiv2023}, we use ISS \cite{Zhong-ICCVW2009} and FPFH \cite{Rusu-ICRA2009} as feature descriptors for the ETH dataset. This time, in \cref{tab:eth_comparison}, more methods run out of memory including both versions of MAC and two deep learning methods, PointDSC and VBReg. This is because, after feature matching, there are still a few relatively large point clouds~(\eg, $N >$ 10k) with high outlier ratios, making registration challenging.

\myparagraph{Analysis of Results} That MAC often runs out of memory is not quite surprising, as it stores all maximal cliques besides the consistency graph. In theory, in the worst case, a graph can have as many as $3^{N/3}$ maximal cliques \cite{Moon-IJM1965,Tomita-TCS2006}. In our experiments of \cref{tab:kitti_comparison}, MAC (Python) typically produces more than $10^4$ cliques, and in \cref{tab:3DMatch_comparison}, MAC (C++) typically produces more than $10^5$, or occasionally a few million, maximal cliques ($N=5000$ in both tables).

That TR-DE \cite{Chen-CVPR2022c} is often slower than $\tear$ validates our design and implementation again: Even though TR-DE is also a branch-and-bound method (similarly to $\tear$), $\tear$ is up to $70$ times faster than TR-DE on real data.

That other methods perform really well---once given the memory they demand---might have been established knowledge in the literature (\eg, TEASER++ \cite{Yang-T-R2021}, PointDSC~\cite{Bai-CVPR2021}, SC$^2$-PCR~\cite{Chen-CVPR2022b}). But the drawback of their unscalability manifests itself when memory is insufficient or point clouds encountered are huge; we re-emphasize this point in \cref{section:huge-scale-experiments}.

Overall, on the three standard datasets, we find $\tear$ to be competitive in registration accuracy, while it runs a few times faster than the second fastest method.




\begin{table}[!t]
    \centering
    \caption{Results on 3DMatch (3DSmoothNet descriptors). }
    \vspace{-0.8em}
    \footnotesize
    \renewcommand{\tabcolsep}{3.2pt}
    \renewcommand\arraystretch{1.25}
    \begin{tabular}{l|ccccc}
        \Xhline{1pt}
         Method & RR($\%$)$\uparrow$ & F1($\%$)$\uparrow$ & RE($^{\circ}$)$\downarrow$ & TE(cm)$\downarrow$  & Time(s)$\downarrow$\\ 
        \Xhline{0.5pt}
        RANSAC~\cite{Fischler-C-ACM1981} & 92.30 & 87.95 & 2.59 & 7.91  & \underline{2.52}\\
        TEASER++~\cite{Yang-T-R2021} & 92.05 & 87.42 & 2.23 & 6.62 & 3.77 \\
        SC$^2$-PCR~\cite{Chen-CVPR2022b} & 94.45 & 89.23 & 2.19 & \textbf{6.40}  & 4.56 \\
        MAC (Python)~\cite{Zhang-CVPR2023} & \multicolumn{5}{c}{out-of-memory} \\
        MAC (C++)~\cite{Zhang-CVPR2023} & \textbf{94.57} & \underline{89.48} & 2.21 & \underline{6.52} & 6.89 \\
        PointDSC~\cite{Bai-CVPR2021} & 93.65 & 89.07 & \underline{2.17} & 6.75  & 5.28 \\
        VBReg~\cite{Jiang-CVPR2023} & 37.09  & 18.07 & 6.15 & 15.65 & 8.07\\
        TR-DE~\cite{Chen-CVPR2022c} & 91.37 & 86.99 & 2.71 & 7.62 & 12.76 \\
        $\tear$ (Ours) & \underline{94.52} & \textbf{89.65} & \textbf{2.06} & 6.55  & \textbf{1.26} \\
        \Xhline{1pt}
    \end{tabular}
    \label{tab:3DMatch_comparison}
\end{table}

\begin{table}[!t]
    \centering
    \caption{Results on the KITTI dataset (FPFH descriptors).}
    \vspace{-0.8em}
    \footnotesize
    \renewcommand{\tabcolsep}{3.2pt}
    \renewcommand\arraystretch{1.25}
    \begin{tabular}{l|ccccc}
        \Xhline{1pt}
         Method & RR($\%$)$\uparrow$ & F1($\%$)$\uparrow$ & RE($^{\circ}$)$\downarrow$ & TE(cm)$\downarrow$ & Time(s)$\downarrow$\\ 
        \Xhline{0.5pt}
        RANSAC~\cite{Fischler-C-ACM1981} & 95.68 & 81.23 & 1.06 & 23.19  & 3.79 \\
        TEASER++~\cite{Yang-T-R2021} & 97.84 & 93.73 & \underline{0.43} & 8.67  & \underline{0.36} \\
        SC$^2$-PCR~\cite{Chen-CVPR2022b} & \textbf{99.64} & \textbf{94.26} & \textbf{0.39} & \textbf{8.29} & 4.33 \\
        MAC (Python)~\cite{Zhang-CVPR2023} & 94.95 & 89.52 & 0.52 & 10.26 & 4.53 \\
        MAC (C++)~\cite{Zhang-CVPR2023} & \multicolumn{5}{c}{out-of-memory} \\
        PointDSC~\cite{Bai-CVPR2021} & 98.20 & 92.71 & 0.57 & 8.67 & 6.20 \\
        VBReg~\cite{Jiang-CVPR2023}  & 98.92 & 92.69  & 0.45 & \underline{8.41} & 8.20\\
        TR-DE~\cite{Chen-CVPR2022c} & 96.76 & 87.20 & 0.90 & 15.63 & 8.66 \\
        $\tear$ (Ours) & \underline{99.10} & \underline{93.85} & \textbf{0.39} & 8.62  & \textbf{0.25} \\
        \Xhline{1pt}
    \end{tabular}
    \label{tab:kitti_comparison}
\end{table}

\begin{table}[!t]
    \centering
    \caption{Results on the ETH dataset (ISS + FPFH descriptors). }
    \vspace{-0.8em}
    \footnotesize
    \renewcommand{\tabcolsep}{3.85pt}
    \renewcommand\arraystretch{1.25}
    \begin{tabular}{l|ccccc}
        \Xhline{1pt}
        Method & RR($\%$)$\uparrow$& F1($\%$)$\uparrow$  & RE($^{\circ}$)$\downarrow$ & TE(cm)$\downarrow$ & Time(s)$\downarrow$\\ 
        \Xhline{0.5pt}
        RANSAC~\cite{Fischler-C-ACM1981} & 69.05 & 65.17 & 0.44 & 10.31 & 6.12\\
        TEASER++~\cite{Yang-T-R2021} & \textbf{96.43} & \underline{92.23} & \underline{0.29} & \underline{5.84} & \underline{0.85} \\
        SC$^2$-PCR~\cite{Chen-CVPR2022b} & \underline{91.67} & 90.34 & 0.32 & 6.25 & 12.93 \\
        MAC (both)~\cite{Zhang-CVPR2023} &  \multicolumn{5}{c}{out-of-memory} \\ 
        PointDSC~\cite{Bai-CVPR2021} & \multicolumn{5}{c}{out-of-memory} \\
        VBReg~\cite{Jiang-CVPR2023}  & \multicolumn{5}{c}{out-of-memory} \\
        TR-DE~\cite{Chen-CVPR2022c} & 88.09 & 73.40 & 0.62 & 16.49  & 7.57 \\
        $\tear$ (Ours) & \textbf{96.43}  & \textbf{93.14} & \textbf{0.25} & \textbf{5.71} & \textbf{0.38} \\
        \Xhline{1pt}
    \end{tabular}
    \label{tab:eth_comparison}
\end{table}

\begin{table*}[!t]
    \centering
    \caption{Experiments with huge-scale point pairs at extremely high outlier ratios  on the Stanford 3D scanning dataset~\cite{Curless-CCGIT1996}. We report rotation errors in degrees $|$ translation errors in centimeter $|$ running times in seconds of various methods. Average over 20 trials.}
    \vspace{-0.8em}
    \footnotesize
    \renewcommand{\tabcolsep}{7pt}
    \renewcommand\arraystretch{1.25}
    \begin{tabular}{c|ccccc}
        \Xhline{1pt}
        \makecell{Point Cloud Name} & \textit{Armadillo} & \textit{Happy Buddha} & \textit{Asian Dragon} & \textit{Thai Statue} & \textit{Lucy}\\ 
        \makecell{\# of Input Point Pairs (Outlier Ratio)} & $10^5$ ($99\%$) & $5\times10^5$ ($99.2\%$) & $10^6$ ($99.4\%$) & $4\times10^6$ ($99.6\%$) & $10^7$ ($99.8\%$)\\ 
        \Xhline{0.5pt}
        \textit{Consistency Graph}-based~\cite{Yang-T-R2021, Chen-CVPR2022b, Zhang-CVPR2023} & out-of-memory & & & &  \\
        \textit{Deep Learning}-based~\cite{Bai-CVPR2021, Jiang-CVPR2023} & out-of-memory & & & &  \\
        RANSAC~\cite{Fischler-C-ACM1981} & 53.1 $|$ 23.1 $|$ 95.9 & 30.7 $|$ 15.8 $|$ 582 & 36.5 $|$ 22.7 $|$ 1179 & 37.1 $|$ 24.6 $|$ 6125 & $\geq$ 8 hours \\
        FGR~\cite{Zhou-ECCV2016} & 57.1 $|$ 39.1 $|$ 2.48 & 84.1 $|$ 23.7 $|$ 19.3 & 62.1 $|$ 19.7 $|$ 39.8 & 79.7 $|$ 15.2 $|$ 175 &  88.9 $|$ 11.5 $|$ 449\\
        GORE~\cite{Bustos-TPAMI2018}  & 0.67 $|$ 0.52 $|$ 6592 & $\geq$ 12 hours & & &  \\ 
        TR-DE~\cite{Chen-CVPR2022c} & 36.5 $|$ 16.9 $|$ 4658 & $\geq$ 9 hours & & & \\
        $\tear$ (Ours) & 0.51 $|$ 0.25 $|$ 12.7 & 0.23 $|$ 0.13 $|$ 119 & 0.14 $|$ 0.12 $|$ 356 & 0.11 $|$ 0.08 $|$ 1013  & 0.07 $|$ 0.06 $|$ 1972 \\
        \Xhline{1pt}
    \end{tabular}
    \label{tab:huge_scale}
\end{table*}

\section{Huge-Scale Experiments}\label{section:huge-scale-experiments}
Inspired by \cite{Peng-CVPR2022}, here we perform \textit{huge-scale} experiments using 5 objects (\textit{Armadillo}, \textit{Happy Buddha},  \textit{Asian Dragon}, \textit{Thai Statue}, \textit{Lucy}) from the Stanford 3D scanning dataset \cite{Curless-CCGIT1996}. The number of points in each object ranges from $10^5$ to $10^7$. For each object, we resize it so that all points lie in $[0,1]^3$ and treat it as the source point cloud $\{\bx_i\}_{i=1}^N$. We use the same procedure of \cref{sec:syn_exper} to generate the target point cloud $\{\by_i\}_{i=1}^N$ with a controlled number of random outliers.

\cref{tab:huge_scale} shows the consistency graph methods and deep learning methods are unscalable, RANSAC \cite{Fischler-C-ACM1981} and FGR \cite{Zhou-ECCV2016} are inaccurate at extreme outlier ratios, GORE \cite{Bustos-TPAMI2018} and TR-DE \cite{Chen-CVPR2022c} are inefficient. \cref{fig:huge-repre} visualizes the results in \cref{tab:huge_scale} for \textit{Asian Dragon} and \textit{Lucy}, where $\tear$ is the only method that aligns the huge-scale point clouds accurately.

\begin{figure}[!t]
    \centering
    \includegraphics[width=0.45\textwidth]{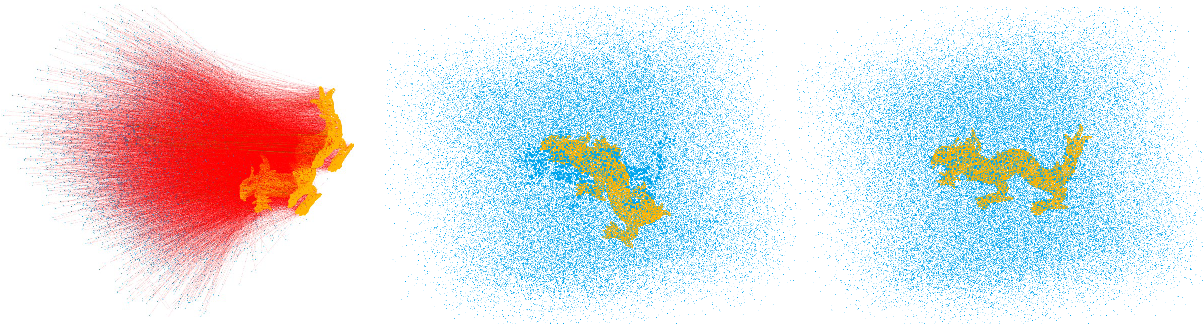}
    \\
    \makebox[0.14\textwidth]{\footnotesize $10^6$ (99.4$\%$)}
    \makebox[0.14\textwidth]{\footnotesize \makecell{RANSAC~\cite{Fischler-C-ACM1981} \\ 51.6$^{\circ}$ $|$ 18.6cm $|$ 1263s}}
    \makebox[0.14\textwidth]{\footnotesize \ \ \quad \makecell{$\tear$ (Ours) \\ 0.13$^{\circ}$ $|$ 0.11cm $|$ 289s}}
    \\[0.2em]
   \includegraphics[width=0.45\textwidth]{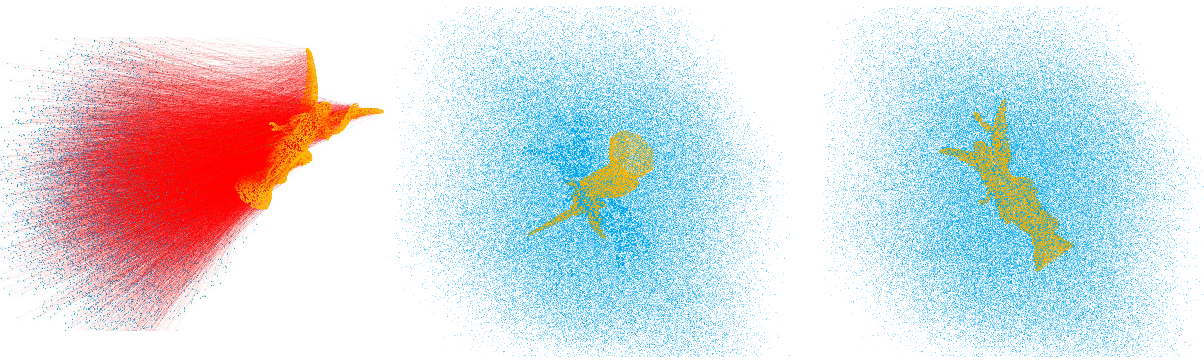}
    \\
    \makebox[0.14\textwidth]{\footnotesize $10^7$ (99.8$\%$)}
    \makebox[0.14\textwidth]{\footnotesize \makecell{FGR~\cite{Zhou-ECCV2016} \\ 91.7$^{\circ}$ $|$ 15.1cm $|$ 514s}}
    \makebox[0.14\textwidth]{\footnotesize \ \ \quad \makecell{$\tear$ (Ours) \\ 0.08$^{\circ}$ $|$ 0.05cm $|$ 1813s}}
    \\[-0.8em]
    \caption{Column 1: input point pairs (Top: \textit{Asian Dragon}, Bottom: \textit{Lucy}). Column 2: Outputs of RANSAC \cite{Fischler-C-ACM1981} and FGR \cite{Zhou-ECCV2016} (Format: rotation error $|$ translation error $|$ running time). Column 3: $\tear$ succeeds in registration for challenging scenarios.}
    \label{fig:huge-repre}
\end{figure}

Inasmuch as downsampling the point clouds would enable other methods to be applied, we perform such an experiment for these methods to compare with $\tear$. In particular, we take the $10^7$ point pairs generated from \textit{Lucy} and downsample them to $10^4$ point pairs, which are given as inputs to other methods. In \cref{fig:huge_down_rota}, the rotation errors of these methods are large (the translation errors are shown in the appendix). Indeed, downsampling would throw away inliers, making the subsequent registration problem more challenging. In fact, in experiments of \cref{fig:huge_down_rota}b, we found it was not just that the total number of inliers inevitably decreased after downsampling; the outlier ratio could even grow from $95\%$ to averagely $98.17\%$. In contrast, since $\tear$ is capable of operating the original point cloud, it delivers lower errors than other methods that perform downsampling. 

\begin{figure}[!t]
    \centering
    \includegraphics[width=0.45\textwidth]{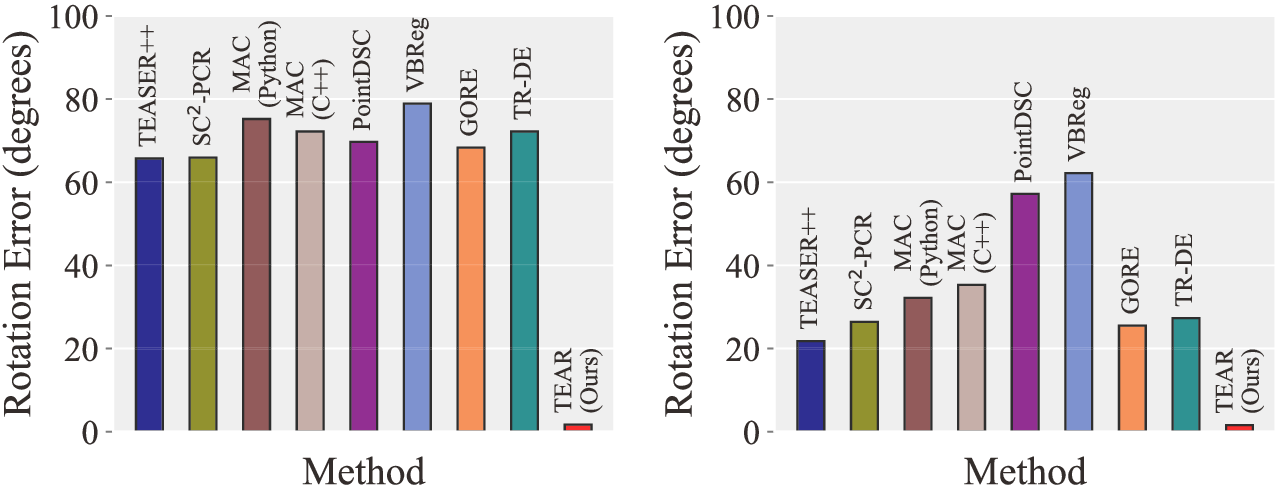}
    \\[-0.5em]
    \makebox[0.235\textwidth]{\footnotesize \quad \quad \quad (a)}
    \makebox[0.235\textwidth]{\footnotesize \quad \quad (b)}
    \\[-0.8em]
    \caption{Average rotation errors of other methods in \cref{tab:huge_scale} taking  as inputs the $10^4$ points downsampled from \textit{Lucy} that originally has $10^7$ point pairs (\cref{fig:huge_down_rota}a: $99.8\%$ outliers; \cref{fig:huge_down_rota}b: $95\%$ outliers). 
    $\tear$ runs on the original $10^7$ input point pairs. 20 trials.}
    \label{fig:huge_down_rota}
\end{figure}

\section{Conclusion and Future Work}
In the paper, we showed that $\tear$ stands on the simple principle of branch-and-bound, stands with state-of-the-art methods at the same level of accuracy, stands in contrast to other slower branch-and-bound methods, and stands out as a scalable method for outlier-robust 3D registration. 

We found it exciting to exhibit a case where branch-and-bound, a technique famously known for its global optimality guarantees and infamously known for its being slow, can actually be competitive in outlier-robust 3D registration. The key ideas of achieving so include using \textit{Truncated Entry-wise Absolute Residuals} \eqref{eq:TEAR} as the robust loss, deriving tight upper and lower bounds based on \ref{eq:TEAR}, and engineering an efficient implementation. We look forward to extending these ideas to other geometric vision problems, for example, absolute pose estimation (2D-3D registration).

\vspace{0.6em}
\noindent \textbf{Acknowledgements.} This work is supported in part by the Shenzhen Portion of Shenzhen-Hong Kong Science and Technology Innovation Cooperation Zone under Grant HZQBKCZYB-20200089, in part by the Hong Kong RGC under T42-409/18-R and 14218322, in part by the Hong Kong Centre for Logistics Robotics, and in part by the VC Fund 4930745 of the CUHK T Stone Robotics Institute.


{
    \small
    \bibliographystyle{ieeenat_fullname}
    \bibliography{Liangzu}
}

\appendix
\onecolumn

\section*{Appendix}

This appendix is organized as follows:
\begin{itemize}
    \item In \cref{sec:BnB}, we review the general branch-and-bound algorithm framework.
    \item In \cref{sec:bounds_tear,section:solve-tear-2}, we introduce the proposed bound computation methods for \textcolor{red}{TEAR-1} and \textcolor{red}{TEAR-2}, respectively.
    \item In \cref{sec:bounds_CM,sec:bounds_TLS}, we introduce the developed bound computation methods for \textcolor{cvprblue}{CM-1} and \textcolor{cvprblue}{TLS-1}, respectively.
    \item In \cref{appen: proof_prop_LB}, we prove \cref{prop:LB-range} (the statement of \cref{prop:LB-range} is presented in \cref{subsec:tear_lb}).
    \item In \cref{sec:extra_exper}, we present extra experimental details.
\end{itemize}

\section{The General Branch-and-bound Algorithm}
\label{sec:BnB}
In this section, we review the branch-and-bound framework. Given an objective function $\mathcal{T}(\cdot)$ and the corresponding decision variable $\mathbf{v}$, branch-and-bound computes a global minimizer of $\mathcal{T}(\cdot)$ up to a prescribed error $\epsilon>0$. It does so by recursively dividing the parameter space $\mathbb{B}_0$ of the decision variable into smaller branches and computing upper and lower bounds of the objective over each branch. By upper bounds we mean upper bounds of the optimal value, and by lower bounds over a given branch we mean lower bounds of the optimal value on this branch.

The algorithmic listing of this branch-and-bound recursion is shown in \cref{algo:bnb_algorithm}. The algorithm maintains the smallest upper bound $U^*$ found so far and the corresponding point $\mathbf{v}^*$ that achieves that bound; $U^*$ and $\mathbf{v}^*$ are initialized at Line 5, by a function that computes upper bounds. In the meantime, the algorithm starts with the initial branch $\mathbb{B}_0$ (Line 3), computes a lower bound (Line 6), and puts the branch $\mathbb{B}_0$ into a priority queue (Line 7). Next, branch-and-bound enters a while loop (Lines 8-24), at each iteration keeping examining the branch $\mathbb{B}$ in the queue with the highest priority (Line 9). If the smallest upper bound $U^*$ found so far is smaller than the lower bound $L(\mathbb{B})$ up to a small tolerance $\epsilon$, then the current point $\mathbf{v}^*$ already obtains the minimum value up to error $\epsilon$, therefore the algorithm terminates (Lines 10-12). Otherwise, the algorithm proceeds by dividing the current branch $\mathbb{B}$ into smaller branches $\mathbb{B}_i$. Similarly, if the lower bound $L(\mathbb{B}_i)$ is larger than $U^*$, then $\mathbb{B}_i$ can never contain any optimal solutions, therefore we can prune it (Lines 15 and 16). Otherwise, the algorithm computes an upper bound for $\mathbb{B}_i$ (Line 18), updates the current solutions $(U^*,\mathbf{v}^*)$ if needed (Lines 19-21), and stores the branch $\mathbb{B}_i$ into the queue $\mathcal{Q}$ with priority $L(\mathbb{B}_i)$.

As \cref{algo:bnb_algorithm} shows, the crucial step in the branch-and-bound method is computing the upper and lower bound of the objective on a given branch $\mathbb{B}$. Hence, to solve \textcolor{red}{TEAR-1}, \textcolor{cvprblue}{CM-1}, and \textcolor{cvprblue}{TLS-1}, we propose methods for computing the desired bounds in \cref{sec:bounds_tear,sec:bounds_CM,sec:bounds_TLS}, respectively. Moreover, in \cref{section:solve-tear-2} we describe the bound computation for \textcolor{red}{TEAR-2}.




\begin{algorithm}[!t]
    \DontPrintSemicolon
    \caption{The general branch-and-bound template to minimize a given objective function globally optimally.}\label{algo:bnb_algorithm}

    \textbf{Input:} Objective function $\mathcal{T}(\cdot)$, threshold $\epsilon$;

    \textbf{Output:} A global minimizer $\mathbf{v}^{*}$;

    Initial branch $\mathbb{B}_0 \gets $ parameter space of $\mathbf{v}$;

    $\mathcal{Q} \gets $ An empty priority queue;

    $U^*,\ \mathbf{v}^{*}\gets $ getUpperBound($\mathcal{T}(\cdot),\ \mathbb{B}_0)$; \tcp*{$U^*$ denotes the smallest upper bound so far, and $U^*=\mathcal{T}(\mathbf{v}^*)$.}

    $L(\mathbb{B}_0) \gets $ getLowerBound($\mathcal{T}(\cdot),\ \mathbb{B}_0)$;

    Insert $\mathbb{B}_0$ into $\mathcal{Q}$ with priority $L(\mathbb{B}_0)$;

    \textbf{while} $\mathcal{Q}$ is not empty \textbf{do}

    \ \ \ \ Pop the branch $\mathbb{B}$ from $Q$ with the lowest lower bound $L(\mathbb{B})$;

    \ \ \ \ \textbf{if} $U^*$ - $L(\mathbb{B}) < \epsilon$ \textbf{then} \tcp*{Terminate the algorithm: The current $\mathbf{v}^*$ is already optimal up to tolerance $\epsilon$.}

    \ \ \ \ \ \ \ \ \textbf{end while}

    \ \ \ \ \textbf{end if}

    \ \ \ \ \textbf{for} $\mathbb{B}_i$ \textbf{in} divideBranch($\mathbb{B}$) \textbf{do} \tcp*{Divide $\mathbb{B}$ into smaller branches $\mathbb{B}_i$'s and compute upper and lower bounds for each $\mathbb{B}_i$. }

    \ \ \ \ \ \ \ \ $L(\mathbb{B}_i) \gets$ getLowerBound($\mathcal{T}(\cdot),\ \mathbb{B}_i)$;

    \ \ \ \ \ \ \ \ \textbf{if} $L(\mathbb{B}_i) \geq U^*$ \textbf{then} \tcp*{The branch $\mathbb{B}_i$ does not contain any optimal solutions, so prune it.}


    \ \ \ \ \ \ \ \ \ \ \ \ \textbf{continue}

    \ \ \ \ \ \ \ \ \textbf{else} \tcp*{The branch $\mathbb{B}_i$ could potentially contain better solutions, so compute an upper bound and store it into the priority queue.}

    \ \ \ \ \ \ \ \ \ \ \ \  $U(\mathbb{B}_i),\ \mathbf{v}_i \gets$ getUpperBound($\mathcal{T}(\cdot),\ \mathbb{B}_i)$;

    \ \ \ \ \ \ \ \ \ \ \ \ \textbf{if} $U(\mathbb{B}) < U^*$ \textbf{then}

    \ \ \ \ \ \ \ \ \ \ \ \ \ \ \ \ $\mathbf{v}^* \gets \mathbf{v}$;\ $U^* \gets U(\mathbb{B})$;

    \ \ \ \ \ \ \ \ \ \ \ \ \textbf{end if}

    \ \ \ \ \ \ \ \ \ \ \ \ Insert $\mathbb{B}_i$ into $\mathcal{Q}$ with priority $L(\mathbb{B}_i)$;

    \ \ \ \ \ \ \ \ \textbf{end if}

    \ \ \ \ \textbf{end for}

    \textbf{end while}
\end{algorithm}

\section{Bounds for \textcolor{red}{TEAR-1}}
\label{sec:bounds_tear}

\subsection{Upper Bound}
\label{subsec:tear_ub}
Given a branch of $\br$, \eg, $[\alpha_1,\alpha_2]\times[\beta_1,\beta_1]$, we choose the center $\dot{\br}_1$ to compute an upper bound $\overline{U}$. Let $a_i:=y_{i1} - \dot{\br}_1^{\top}\bx_i$, and an upper bound can be computed as
\begin{equation}\label{eq:UB_r1_t1}
    \begin{split}
        \overline{U} &= \min_{t_1\in\mathbb{R}} \sum_{i=1}^{N} \min\left\{| a_i - t_1 |,\ \xi_{i1}\right\} \\
        &= \min_{t_1 \in \mathbb{R}} U(t_1),
    \end{split}
\end{equation}
where $U(t_1)$ is defined to be the sum of $U(t_1,i):=\min\left\{| a_i - t_1 |,\ \xi_{i1}\right\}$, that is $U(t_1):=\sum_{i=1}^N U(t_1,i)$. Define $l_{ai}:=a_1 - \xi_{i1}$ and $u_{ai}:=a_1 + \xi_{i1}$. Note that $U(t_1,i)$ equals to $| a_i - t_1 |$ if $t_1 \in [l_{ai}, u_{ai}]$, or otherwise $\xi_{i1}$~(see \cref{fig:UB_illus}a). Furthermore, $U(t_1,i)$ is differentiable at any point $t_1$ except when $t_1$ is equal to $l_{ai}$, $a_i$, or $u_{ai}$. Finally, at differentiable points, the derivative of $U(t_1,i)$ is either $1$ or $-1$ or $0$, therefore the derivative of $U(t_1)$ lies in $\{-N,\dots,-1, 0,1,\dots,N\}$.

We need more notations. Let $\{\lambda_{k}\}_{k=1}^{3N}$ be a sorted version of the 3$N$ numbers $\{a_i\}_{i=1}^{N} \cup \{l_{ai}\}_{i=1}^{N} \cup \{u_{ai}\}_{i=1}^{N}$. 
For each $k = 1, \cdots, 3N - 1$, define the open interval $\cI_k := (\lambda_k, \lambda_{k+1})$. Without loss of generality, we assume $\lambda_k\neq \lambda_{k+1}$ and therefore $\cI_k$ is not empty. Note that the derivative of $U(t_1,i)$ is equal to $1$ or $-1$ depending on the sign of $a_i-t_1$. Since the sign of $a_i-t_1$ in the definition of $U(t_1,i)$ is constant on every interval $\cI_k$, the derivative of $U(t_1,i)$ is constant on $\cI_k$ for every $i$ and every $k$. As a consequence, the derivative of $U(t_1)$ is constant on every interval $\cI_k$~(see \cref{fig:UB_illus}b). 

\begin{figure}[!t]
    \centering
    \includegraphics[width=0.75\textwidth]{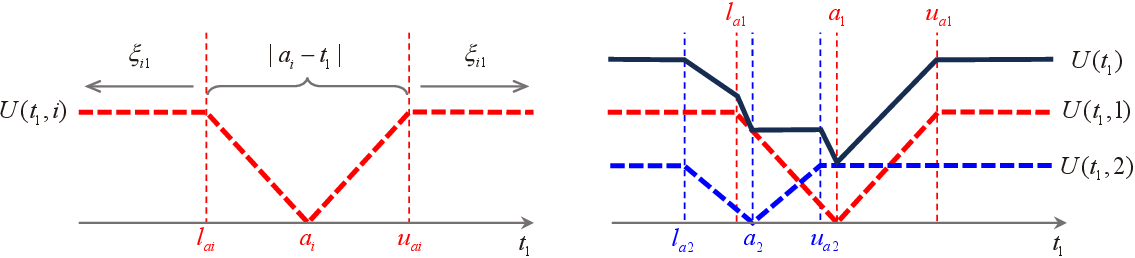}
    \\
    \makebox[0.3\textwidth]{\footnotesize (a) \quad \quad}
    \makebox[0.3\textwidth]{\footnotesize \quad \quad (b)}
    \\[-0.5em]
    \caption{(a) Illustration of $U(t_1,i) = \min \{ | a_i - t_1 |,\ \xi_{i1} \}$; (b) Illustration of $U(t_1)=\sum_{i=1}^{N} U(t_1,i)$ when $N=2$. (\cf \cref{subsec:tear_ub}).}
    \label{fig:UB_illus}
\end{figure}

With the above notations, we derive the following proposition that localizes a global minimizer of \eqref{eq:UB_r1_t1}:
\begin{proposition}\label{prop:UB-endpoints}
    In the $N$ points $\{a_i\}_{i=1}^{N}$, there is a globally optimal solution of \eqref{eq:UB_r1_t1}.
\end{proposition}
\begin{proof}

    First we will show, in the $3N$ points $\{\lambda_{k}\}_{k=1}^{3N}$, there is a globally optimal solution of \eqref{eq:UB_r1_t1}. To prove this, let us assume there is a global minimizer different from the $3N$ points $\{\lambda_{k}\}_{k=1}^{3N}$. Note that $U(t_1)$ in \eqref{eq:UB_r1_t1} is differentiable except at the $3N$ points $\{\lambda_{k}\}_{k=1}^{3N}$, therefore the derivative of $U(t_1)$ at this global minimizer is $0$. Moreover, this derivative is $0$ on the entire interval $\cI_k$ and the objective value $U(t_1)$ is constant on $\cI_k$. Thus, every point on $\cI_k=(\lambda_k, \lambda_{k+1})$ is a global minimizer. Since $U(t_1)$ is continuous on $[\lambda_k, \lambda_{k+1}]$, has derivative $0$ on $(\lambda_k, \lambda_{k+1})$, and every point on $(\lambda_k, \lambda_{k+1})$ is a global minimizer, we conclude that $\lambda_k$ and $\lambda_{k+1}$ are global minimizers of $U(t_1)$ as well. We have just shown that in the $3N$ points $\{\lambda_{k}\}_{k=1}^{3N}$, there is a globally optimal solution of \eqref{eq:UB_r1_t1}.

    Next, we take one step further and prove, in the $N$ points $\{a_i\}_{i=1}^{N}$, there is a globally optimal solution of \eqref{eq:UB_r1_t1}. Let us assume $\lambda_{k+1}$ is a global minimizer of $U(t_1)$. We need to show $\lambda_{k+1}\in \{a_i\}_{i=1}^{N}$. For the sake of contradiction, assume $\lambda_{k+1} \in \{l_{ai}\}_{i=1}^{N}\cup \{u_{ai}\}_{i=1}^{N}$. Let us analyze the two intervals associated with $\lambda_{k+1}$, namely $\cI_k = (\lambda_k, \lambda_{k+1})$ and $\cI_{k+1} = (\lambda_{k+1}, \lambda_{k+2})$, and let us also analyze the derivatives of $U(t_1)$ on them. Since the derivatives of $U(t_1,i)$ on them are constant, let us denote by $U'(\cI_{k+1},i)$ and $U'(\cI_{k},i)$, respectively, the derivatives of $U(t_1,i)$ at some point of $\cI_{k+1}$ and $\cI_{k}$. Define $U'(\cI_{k+1})$ and $U'(\cI_k)$ similarly. Then we have the following observations:
    \begin{itemize}
        \item If $\lambda_{k+1} = l_{ai}$ for some $i$, then $U'(\cI_{k},i)=0$ and  $U'(\cI_{k+1},i)=-1$, \ie, by going from $\cI_k$ to $\cI_{k+1}$, the $i$-th component function $U(t_1,i)$ enters the linear decrease regime from the constant regime. At the same time, we have $U'(\cI_k, j)= U'(\cI_{k+1}, j)$ for any $j\neq i$; in other words, $U(\cI_k,j)$ keeps the same trend (linearly decrease, linearly increase, or constant). As a result, we have $U'(\cI_k)=U'(\cI_{k+1}) +1$.
        \item If $\lambda_{k+1} = u_{ai}$ for some $i$, then $U'(\cI_{k},i)=1$ and  $U'(\cI_{k+1},i)=0$. And we have $U'(\cI_k, j)= U'(\cI_{k+1}, j)$ for any $j\neq i$. Similarly we have $U'(\cI_k)=U'(\cI_{k+1}) +1$.
    \end{itemize}
    Since $U'(\cI_k)$ and $U'(\cI_{k+1})$ lie in $\{-N,\dots,-1, 0,1,\dots,N\}$, the assumption $\lambda_{k+1} \in \{l_{ai}\}_{i=1}^{N}\cup \{u_{ai}\}_{i=1}^{N}$ leads us to the following cases:
    \begin{itemize}
        \item If $U'(\cI_k)\geq 1$, then we have $U(\lambda_{k+1})>U(\lambda_k)$.
        \item If $U'(\cI_k)\leq 0$, then $U'(\cI_{k+1})\leq -1$ and we have $U(\lambda_{k+1})>U(\lambda_{k+2})$.
    \end{itemize}
    We have shown $U(\lambda_{k+1}) > \min\{U(\lambda_k), U(\lambda_{k+2})\}$. This is a contradiction to the assumption $\lambda_{k+1} \in \{l_{ai}\}_{i=1}^{N}\cup \{u_{ai}\}_{i=1}^{N}$. Therefore, any global minimizer of $U(t_1)$ in $\{\lambda_{k}\}_{k=1}^{3N}$ must be a point of $\{ a_i \}_{i=1}^N$. The proof is complete.
\end{proof}

\begin{algorithm}[!t]
    \DontPrintSemicolon
    \textbf{Input:} $\{a_i\}_{i=1}^{N}$, $\{l_{ai}\}_{i=1}^{N}$, and $\{u_{ai}\}_{i=1}^{N}$;

    \textbf{Output:} A global minimizer ${t_1}'$ and the minimum value $\overline{U}$ of \eqref{eq:UB_r1_t1};

    $\{\lambda_k\}_{k=1}^{3N} \gets $ Sort $\{a_i\}_{i=1}^{N} \cup \{l_{ai}\}_{i=1}^{N} \cup \{u_{ai}\}_{i=1}^{N}$;
        
        $U(\lambda_1) \gets \sum_{i=1}^{N} \xi_{i1}$;\quad $\overline{U} \gets U(\lambda_1)$; \tcp*{ Use $\overline{U}$ to store the smallest upper bound so far.}

        $m_0^{u} \gets 0 $;\quad $m_0^{d} \gets 0$; 
 
	\textbf{for} $k\gets 1$ \textbf{to} $3N - 1$ \textbf{do}:

 \ \ \ \ \textbf{if} $\lambda_{k} \in \{a_i\}_{i=1}^{N}$ \textbf{then}

        \ \ \ \ \ \ \ \ $m_k^{d} \gets m_{k-1}^{d} - 1$;
        
        \ \ \ \ \ \ \ \ $m_k^{u} \gets m_{k-1}^{u} + 1$;

        \ \ \ \ \ \ \ \ \textbf{if} $U(\lambda_k) < \overline{U}$ \textbf{then} \tcp*{Update $\overline{U}$ and the minimizer ${t_1}'$ once $U(\lambda_k)$ is smaller than $ \overline{U}$.}
        
        \ \ \ \ \ \ \ \ \ \ \ \ ${t_1}' \gets \lambda_k$;\quad $\overline{U} \gets U(\lambda_k)$;

        \ \ \ \ \ \ \ \ \textbf{end if}
        
        \ \ \ \ \textbf{else if} $\lambda_{k} \in \{l_{ai}\}_{i=1}^{N}$ \textbf{then}

        \ \ \ \ \ \ \ \ $m_k^{d} \gets m_{k-1}^{d} + 1$;

        \ \ \ \ \textbf{else if} $\lambda_{k} \in \{u_{ai}\}_{i=1}^{N}$ \textbf{then}

        \ \ \ \ \ \ \ \ $m_k^{u} \gets m_{k-1}^{u} - 1$;

        \ \ \ \ \textbf{end if}

        \ \  \ \ $U(\lambda_{k+1}) \gets U(\lambda_k)$ + $(m_k^{u} - m_k^{d})(\lambda_{k+1} - \lambda_k)$; \tcp*{Compute $U(\lambda_{k+1})$ incrementally with the updated $m_k^u$ and $m_k^d$ based on \eqref{eq:U_k+1}.}

        \textbf{end for}
        
	\caption{Globally Optimal 1D TEAR Solver for Upper Bound Computation (\cref{prop:UB-endpoints}).} \label{algo:UB-$t_1$}
\end{algorithm}

In light of \cref{prop:UB-endpoints}, an intuitive algorithm to solve~\eqref{eq:UB_r1_t1} would be computing $U(a_i)$ for each $i \in \{1,\dots,N\}$ separately; the solution is some $a_i$ that gives the smallest objective value. This algorithm has $O({N}^2)$ time complexity. 

Let us show how to solve \eqref{eq:UB_r1_t1} in $O(N\log N)$ time. First recall that, on each interval $[\lambda_k, \lambda_{k+1}]$, $U(t_1,i)$ is monotonic or constant. Let us define $\mathcal{M}_k^{u}$ (\textit{resp.}  $\mathcal{M}_k^{d}$) to be the set of indices $i$ for which $U(t_1,i)$ is increasing (\textit{resp.} decreasing). Denote by $m_k^{u}$ and $m_k^{d}$ the respective sizes of $\mathcal{M}_k^{u}$ and  $\mathcal{M}_k^{d}$. 
Then we can compute the value of $U(\lambda_k)$ as follows:
\begin{align}
    U(\lambda_k) &= \sum_{i\in\mathcal{M}_k^{u}}(\lambda_k - a_i) + \sum_{i\in\mathcal{M}_k^{d}}(a_i - \lambda_k) + \sum_{i\in \{1,\cdots,N\}\setminus\{\mathcal{M}_k^{u} \cup \mathcal{M}_k^{d}\}} \xi_{i1}.
\end{align}
An important observation is that, for every $k$, the value $U(\lambda_{k+1})$ can be computed incrementally from $U(\lambda_k)$:
\begin{align}
    U(\lambda_{k+1}) &= \sum_{i\in\mathcal{M}_k^{u}}(\lambda_{k+1} - a_i) + \sum_{i\in\mathcal{M}_k^{d}}(a_i - \lambda_{k+1}) + \sum_{i\in \{1,\cdots,N\}\setminus\{\mathcal{M}_k^{u} \cup \mathcal{M}_k^{d}\}}\xi_{i1}\\
                &= U(\lambda_k) + (m_k^u - m_k^d)(\lambda_{k+1} - \lambda_k). \label{eq:U_k+1}
\end{align}
\cref{algo:UB-$t_1$} implements the recurrence \eqref{eq:U_k+1} and solves \eqref{eq:UB_r1_t1} to global optimality. Interestingly, this recurrent formulation does not depend on the index sets $\mathcal{M}_k^{u}$ and $\mathcal{M}_k^{d}$, so we only need to update $m_k^u$ and $m_k^d$ in \cref{algo:UB-$t_1$}. Finally, since \cref{algo:UB-$t_1$} only consists of a sorting and scanning operation, its time complexity is $O(N \log N)$.


\begin{figure}[!t]
    \centering
    \includegraphics[width=0.75\textwidth]{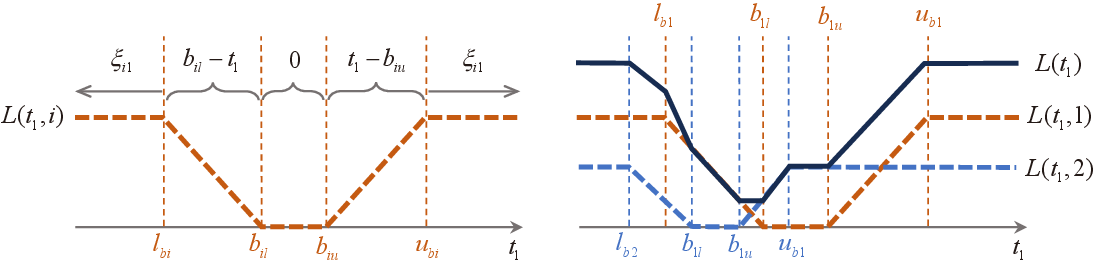}
    \\
    \makebox[0.3\textwidth]{\footnotesize (a) \quad \quad}
    \makebox[0.3\textwidth]{\footnotesize \quad \quad (b)}
    \\[-0.5em]
    \caption{(a) Illustration of $L(t_1,i) = \min_{b_i \in [b_{il}, b_{iu}]} L(t_1, b_i, i)$, where $L(t_1, b_i, i)$ = $\min\{|b_i-t_1|,\ \xi_{i1}\}$; (b) Illustration of $L(t_1)=\sum_{i=1}^{N} L(t_1,i)$ when $N=2$. (\cf \cref{subsec:tear_lb}).}
    \label{fig:LB_illus}
\end{figure}

\subsection{Lower Bound}
\label{subsec:tear_lb}

Here we first consider the following proposition:
\begin{proposition}\label{prop:LB-range}
    Given a branch $[\alpha_1,\alpha_2]\times[\beta_1,\beta_1]$ of $\br_1$, let $[b_{il},b_{iu}]$ denote the range of $b_i:=y_{i1} - {\br_1}^{\top}\bx_i$, specifically:
    \begin{align}
        b_{il} = \min \left\{y_{i1} - {\br_1}^{\top}\bx_i: \br_1=[\sin\beta \cos\alpha, \sin\beta \sin \alpha, \cos\beta ], \ \alpha\in [\alpha_1,\alpha_2],\ \beta\in [\beta_1,\beta_2] \right\}, \\
        b_{iu} = \max \left\{y_{i1} - {\br_1}^{\top}\bx_i: \br_1=[\sin\beta \cos\alpha, \sin\beta \sin \alpha, \cos\beta ], \ \alpha\in [\alpha_1,\alpha_2],\ \beta\in [\beta_1,\beta_2] \right\}.
    \end{align}
    We can compute $b_{il}$ and $b_{iu}$ in constant time.
\end{proposition}
With \cref{prop:LB-range} proved in \cref{appen: proof_prop_LB}, we can compute the bounds $\{b_{il},b_{iu}\}_{i=1}^N$ in $O(N)$ time. Moreover, \cref{prop:LB-range} provides the idea of relaxing (4) in the main manuscript into the following problem:
\begin{equation}
    \begin{split}\label{eq:LB_bi_t1}
        \underline{L} &= \min_{t_1\in\mathbb{R},\ b_i\in[b_{il}, b_{iu}]} \sum_{i=1}^{N} \min \{ | b_i - t_1 |,\ \xi_{i1}\} \\
        & = \min_{t_1\in\mathbb{R},\ b_i\in[b_{il}, b_{iu}]} \sum_{i=1}^{N} L(t_1,b_i,i),
    \end{split} 
\end{equation}
where $L(t_1,b_i,i) := \min \{ | b_i - t_1 |, \ \xi_{i1}\}$. Note that for any given $t_1 \in [b_{il}, b_{iu}]$, we can always set $b_i$ to $t_1$ so that $L(t_1,b_i,i)$ is minimized to $0$. Now let us consider
\begin{equation}
    \begin{split}
        L(t_1,i) := 
        &\min_{b_i\in[b_{il}, b_{iu}]}{L(t_1,b_i,i)} \\
        = &\begin{cases}
            b_{il} - t_1, \quad& t_1 \in [b_{il} - \xi_{i1}, b_{il}]; \\
            0, \quad&t_1 \in [b_{il}, b_{iu}]; \\
            t_1 - b_{iu}, \quad& t_1 \in [b_{iu}, b_{iu} + \xi_{i1}]; \\
            \xi_{i1}, \quad& \textnormal{otherwise}.
        \end{cases}
    \end{split}
\end{equation}
Define $l_{bi} := b_{il} - \xi_{i1}$ and $u_{bi} := b_{iu} + \xi_{i1}$. Note that $L(t_1,i)$ is differentiable at any point $t_1$ except when $t_1$ is equal to $l_{bi}$, $b_{il}$, $b_{iu}$, or $u_{bi}$. Furthermore, at differentiable points, the derivative of $L(t_1,i)$ is either $1$ or $-1$ or $0$~(see \cref{fig:LB_illus}a). Now we can convert \eqref{eq:LB_bi_t1} into the following equivalent problem:
\begin{equation}\label{eq:LB_t1}
    \begin{split}
        \underline{L} &= \min_{t_1\in\mathbb{R}} \sum_{i=1}^{N} \min_{b_i\in[b_{il}, b_{iu}]} L(t_1, b_i, i) \\
        &= \min_{t_1\in\mathbb{R}} \sum_{i=1}^{N} L(t_1,i) \\
        &= \min_{t_1\in\mathbb{R}} L(t_1),
    \end{split}
\end{equation}
where $L(t_1)$ is defined to be the sum of $L(t_1, i)$, i.e., $L(t_1) = \sum_{i=1}^{N} L(t_1,i)$. According to the derivation property of $L(t_1,i)$, the derivative of $L(t_1)$ at differentiable points lies in $\{-N,\dots,-1, 0,1,\dots,N\}$. Based on the above notions, we have

\begin{proposition}\label{prop:LB-endpoints}
    In the $2N$ points $\{b_{il}\}_{i=1}^{N} \cup \{b_{iu}\}_{i=1}^{N}$, there is a globally optimal solution of \eqref{eq:LB_t1}. And \eqref{eq:LB_t1} can be solved by \cref{algo:LB-$t_1$} in $O(N\log N)$ time.
\end{proposition}

\begin{algorithm}[!t]
    \DontPrintSemicolon
    \textbf{Input:} $\{b_{il}\}_{i=1}^{N}$, $\{b_{iu}\}_{i=1}^{N}$, $\{l_{bi}\}_{i=1}^{N}$, and $\{u_{bi}\}_{i=1}^{N}$;

    \textbf{Output:} A global minimizer $\hat{t}_1$ and the minimum value $\underline{L}$ of \eqref{eq:LB_t1};

    $\{\psi_k\}_{k=1}^{4N} \gets $ Sort $\{b_{il}\}_{i=1}^{N} \cup \{b_{iu}\}_{i=1}^{N} \cup \{l_{bi}\}_{i=1}^{N} \cup \{u_{bi}\}_{i=1}^{N}$;
        
        $L(\psi_1) \gets \sum_{i=1}^{N} \xi_{i1}$;\quad $\underline{L} \gets L(\psi_1)$; \tcp*{Use $\underline{L}$ to store the smallest lower bound so far.}

        $n_0^{u} \gets 0 $;\quad $n_0^{d} \gets 0$;
 
	\textbf{for} $k\gets 1$ \textbf{to} $4N - 1$ \textbf{do}:

        \ \ \ \ \textbf{if} $\psi_{k} \in \{b_{il}\}_{i=1}^{N}$ \textbf{then}

        \ \ \ \ \ \ \ \ $n_k^{d} \gets n_{k-1}^{d} - 1$;

        \ \ \ \ \ \ \ \ \textbf{if} $L(\psi_k) < \underline{L}$ \textbf{then} \tcp*{Update $\underline{L}$ and the minimizer $\hat{t}_1$ once $L(\psi_k)$ is smaller than $ \underline{L}$.}
        
        \ \ \ \ \ \ \ \ \ \ \ \ $\hat{t}_1 \gets \psi_k$;\quad $\underline{L} \gets L(\psi_k)$;

        \ \ \ \ \ \ \ \ \textbf{end if}

        \ \ \ \ \textbf{else if} $\psi_{k} \in \{b_{iu}\}_{i=1}^{N}$ \textbf{then}
      
        \ \ \ \ \ \ \ \ $n_k^{u} \gets n_{k-1}^{u} + 1$;

        \ \ \ \ \ \ \ \ \textbf{if} $L(\psi_k) < \underline{L}$ \textbf{then} \tcp*{Similarly to Line 9, update $\underline{L}$ and the minimizer $\hat{t}_1$ once $L(\psi_k)$ is smaller than $ \underline{L}$.}
        
        \ \ \ \ \ \ \ \ \ \ \ \ $\hat{t}_1 \gets \psi_k$;\quad $\underline{L} \gets L(\psi_k)$;

        \ \ \ \ \ \ \ \ \textbf{end if}
        
        \ \ \ \ \textbf{else if} $\psi_{k} \in \{l_{bi}\}_{i=1}^{N}$ \textbf{then}

        \ \ \ \ \ \ \ \ $n_{k+1}^{d} \gets n_k^{d} + 1$;

        \ \ \ \ \textbf{else if} $\psi_{k} \in \{u_{bi}\}_{i=1}^{N}$ \textbf{then}

        \ \ \ \ \ \ \ \ $n_{k+1}^{u} \gets n_k^{u} - 1$;

        \ \ \ \ \textbf{end if}

        \ \  \ \ $L(\psi_{k+1}) = L(\psi_k) + (n_k^u-n_k^d)(\psi_{k+1} - \psi_{k})$; \tcp*{Compute $L(\psi_{k+1})$ incrementally with the updated $n_k^u$ and $n_k^d$ based on \eqref{eq:L_k+1}.}

        \textbf{end for}
        
	\caption{Globally Optimal 1D TEAR Solver for Lower Bound Computation (\cref{prop:LB-endpoints}).} \label{algo:LB-$t_1$}
\end{algorithm}

\begin{proof}
     Let $\{\psi_k\}_{k=1}^{4N}$ be a sorted version of $\{b_{il}\}_{i=1}^{N} \cup \{b_{iu}\}_{i=1}^{N} \cup \{l_{bi}\}_{i=1}^{N} \cup \{u_{bi}\}_{i=1}^{N}$. 
     For each $k = 1, \cdots, 4N$, define the open interval $\cJ_k := (\psi_k, \psi_{k+1})$.
     Without loss of generality, we assume $\psi_k\neq \psi_{k+1}$ and therefore $\cJ_k$ is not empty.
     Similarly to $U(t_1, i)$ in \cref{subsec:tear_ub}, the derivative of $L(t_1,i)$ on every interval $\cJ_k$ is constant~(say, $0, -1$, or $1$); therefore the derivative of $L(t_1)$ is also constant on every interval $\cJ_k$ similarly to $U(t_1)$ in \cref{subsec:tear_ub}~(see \cref{fig:LB_illus}b). Since $L(t_1)$ shares similar derivation properties with $U(t_1)$, we can easily infer that in the $4N$ endpoints $\{\psi_k\}_{k=1}^{4N}$, there is a globally optimal solution of \eqref{eq:LB_t1}.
     
     In the following we take one step further and prove that in the $2N$ points $\{b_{il}\}_{i=1}^{N} \cup \{b_{iu}\}_{i=1}^{N}$, there is a globally optimal solution of \eqref{eq:LB_t1}. 
     Assume that $\psi_{k+1}$ is a global minimizer of $L(t_1)$. We need to show $\psi_{k+1} \notin \{l_{bi}\}_{i=1}^{N} \cup \{u_{bi}\}_{i=1}^{N}$. For the sake of contradiction, suppose that $\psi_{k+1} \in \{l_{bi}\}_{i=1}^{N} \cup \{u_{bi}\}_{i=1}^{N}$. Consider the two intervals associated with $\psi_{k+1}$, namely $\cJ_k = (\psi_k, \psi_{k+1})$ and $\cJ_{k+1} = (\psi_{k+1}, \psi_{k+2})$. Since the derivatives of $L(t_1)$ on them are constant, let us denote by $L'(\cI_{k+1})$ and $L'(\cI_k)$, respectively, the derivatives of $L(t_1)$ at some point of $\cJ_{k+1}$ and $\cJ_{k}$. Then we can observe that, as long as $\psi_{k+1}$ belongs to $\{l_{bi}\}_{i=1}^{N}$ or $\{u_{bi}\}_{i=1}^{N}$, it always holds that $L'(\cJ_k) = L'(\cJ_{k+1}) + 1$~(similarly to the equality $U'(\cI_k)=U'(\cI_{k+1}) + 1$ in \cref{subsec:tear_ub}).
     Since $L'(\cI_{k+1})$ and $L'(\cI_k)$ also lie in $\{-N,\dots,-1, 0,1,\dots,N\}$, the supposition $\psi_{k+1} \in \{l_{bi}\}_{i=1}^{N} \cup \{u_{bi}\}_{i=1}^{N}$ lead to the following cases:
     \begin{itemize}
        \item If $L'(\cJ_k)\geq 1$, then we have $L(\psi_{k+1})>L(\psi_k)$.
        \item If $L'(\cJ_{k})\leq 0$, then $L'(\cI_{k+1})\leq -1$ and we have $L(\psi_{k+1})>L(\psi_{k+2})$.
    \end{itemize}
    We have shown $L(\psi_{k+1}) > \min\{L(\psi_k),L(\psi_{k+2})\}$. This is a contradiction to the assumption that $\psi_{k+1}$ is a global minimizer. Therefore, any global minimizer of $L(t_1)$ in $\{\psi_{k}\}_{k=1}^{4N}$ must be a point of $\{b_{il}\}_{i=1}^{N} \cup \{b_{iu}\}_{i=1}^{N}$.

    Now, to find a global minimizer of \eqref{eq:LB_t1}, it remains to evaluate the objective value of \eqref{eq:LB_t1} at all points $\{b_{il}\}_{i=1}^{N} \cup \{b_{iu}\}_{i=1}^{N}$ and pick one point that yields the minimum objective. Similarly to the computation of $U(\lambda_k)$ in \eqref{eq:U_k+1}, we can compute $ L(\psi_k)$ incrementally using the formula
    \begin{equation}\label{eq:L_k+1}
        \begin{split}
            L(\psi_{k+1}) = L(\psi_k) + (n_k^u-n_k^d)(\psi_{k+1} - \psi_{k}),
        \end{split}
    \end{equation}
    where $n_k^u$ and $n_k^d$ represent the numbers of increasing and decreasing terms of $\{L(t_1, i)\}_{i=1}^N$ at $\cJ_k$, respectively. The details are in \cref{algo:LB-$t_1$}. In particular, \cref{algo:LB-$t_1$} iterates over $\{\psi_k\}_{k=1}^{4N}$ to update $n_k^u$ and $n_k^d$, so as to compute $L(\psi_{k+1})$ based on \eqref{eq:L_k+1}. At each iteration we can update $\underline{L}$ by comparing $L(\psi_{k+1})$ and $\underline{L}$, and when the loop finishes, the minimum value $\underline{L}$ is discovered, associated with the optimal solution stored in $\hat{t}_1$. 
\end{proof}

\section{Bounds for \textcolor{red}{TEAR-2}}\label{section:solve-tear-2}
In this section, we show that we can conduct a simple reparameterization on \textcolor{red}{TEAR-2}, so as to solve it using similar bound computation methods that we develop for \textcolor{red}{TEAR-1}~(see \cref{sec:bounds_tear}). 
Note that the extra constraint $\br_2^\top \hat{\br}_1 = 0$ in \textcolor{red}{TEAR-2} restricts the $\br_2$ to be determined by only 1 degree of freedom. Then we can substitute the decision variable $\br_2$ in \textcolor{red}{TEAR-2} with the variable $\alpha \in [0, 2\pi)$ in the constraint $\br_2 = [\sin{\beta}\cos{\alpha}, \sin{\beta}\sin{\alpha}, \cos{\beta}]^{\top} \in \mathbb{S}^2$, i.e., 
\begin{equation}
    \begin{split}\label{eq:tear-2-trans2}
        \min_{\alpha \in [0, 2\pi),t_2 \in \bbR} & \sum_{i\in \hat{\cI}_1 } \min \{| y_{i2} - \br_2^\top \bx_i - t_2|, \ \xi_{i2} \} \\ 
        \textnormal{s.t.} &\quad \quad \quad \quad \br_2^\top \hat{\br}_1 = 0 \\
        & \br_2 = [\sin{\beta}\cos{\alpha}, \sin{\beta}\sin{\alpha}, \cos{\beta}]^{\top} 
    \end{split}
\end{equation}
Therefore, we can branch over the 1-dimension region $[0, 2\pi)$ to search for the globally optimal solution. Specifically, given a branch $[\alpha_1, \alpha_2]$, we consider the following upper and lower bounds:

\myparagraph{Upper Bound} 
We choose the center $\dot{\alpha}$ of $[\alpha_1, \alpha_2]$ to compute an upper bound. Then the $\br_2$ can be solved by combining the constraints in \eqref{eq:tear-2-trans2} and we denote it by $\dot{\br_2}$. Let $a_{i2} := y_{i2} - \dot{\br_2}^\top \bx_i$, the upper bound can be computed as follows:
\begin{equation}
    \begin{split}\label{eq:tear-2-UB}
        \min_{t_2 \in \bbR } & \sum_{i\in \hat{\cI}_1 } \min \{| a_{i2} - t_2|, \ \xi_{i2}\}.
    \end{split}
\end{equation}
Obviously, \eqref{eq:tear-2-UB} shares the same problem structure with \eqref{eq:UB_r1_t1}. Therefore we can directly call \cref{algo:UB-$t_1$} to get a globally optimal solution of \eqref{eq:tear-2-UB}.

\myparagraph{Lower Bound}
Based on the constraints on $\br_2$ in \eqref{eq:tear-2-trans2}, it is clear that $\beta$ is restricted to an available range on the given branch $[\alpha_1, \alpha_2]$. Let $[\beta_1, \beta_2]$ denotes this range and $b_{i2} := y_{i2} - \br_2^{\top} \bx_i$, we can compute the range $[b_{i2}^l, b_{i2}^u]$ of $b_{i2}$ based on \cref{prop:LB-range}. Then it suffices to relax \eqref{eq:tear-2-trans2} into the following problem for a lower bound:
\begin{equation}
    \begin{split}\label{eq:tear-2-LB}
        \min_{t_2 \in \bbR, b_{i2} \in [b_{i2}^l, b_{i2}^u] } & \sum_{i\in \hat{\cI}_1 } \min \{| b_{i2} - t_2|, \ \xi_{i2}\}. 
    \end{split}
\end{equation}
Note that \eqref{eq:tear-2-LB} and \eqref{eq:LB_bi_t1} share the same structure, therefore we can adopt \cref{algo:LB-$t_1$} to solve \eqref{eq:tear-2-LB} with global optimality.

\section{Bounds for \textcolor{cvprblue}{CM-1}}
\label{sec:bounds_CM}
Before introducing the bounds for \textcolor{cvprblue}{CM-1}, we first present a relevant and widely used algorithm, called \textit{interval stabbing}.

\begin{theorem}
    (\textbf{Interval Stabbing}) Given a set of intervals $\mathcal{Z} = \{[z_{pl}, z_{pu}]\}_{p=1}^P$, consider to find some stabber $z$ that interacts with the most number of intervals, i.e., 
    \begin{equation}\label{eq:interval_stab}
        \begin{split}
            \max_{z\in \mathbb{R}} & \quad \sum_{p=1}^P \bm{1} \left( z\in[z_{pl}, z_{pu}] \right),
        \end{split}
    \end{equation}
    where $\bm{1}(\cdot)$ denotes the indicator function. Then \cref{algo:inter_stab} solves \eqref{eq:interval_stab} in $O(P\log P)$ time.
    \label{theorem:inter_stab}
\end{theorem}

\begin{algorithm}[!t]
    \DontPrintSemicolon
    \textbf{Input:} $\mathcal{Z} = \{[z_{pl},\ z_{pu}]\}_{p=1}^P$;

    \textbf{Output:} Best stabber $z^{*}$ and the corresponding number of stabbed intervals $T^{*}$;

    $\hat{\mathcal{Z}} \gets $ Sort all the endpoints in $\mathcal{Z}$;

    $count \gets 0$;\quad $T^{*} \gets count$;  

    \textbf{for} $p\gets 1$ to $2P$ \textbf{do}:

    \ \ \ \ \textbf{if}  $\hat{\mathcal{Z}}(p)$ is a left endpoint \textbf{then}

    \ \ \ \ \ \ \ \ $count \gets count + 1$;   \tcp*{If the current stabber moves to a left endpoint, one more interval is stabbed.}

    \ \ \ \ \ \ \ \ \textbf{if} $count > T^*$ \textbf{do}:  \tcp*{Update $T^*$ and the best stabber $z^*$ once the current number of stabbed intervals $count$ is larger than $T^*$.}

    \ \ \ \ \ \ \ \ \ \ \ \ $z^* \gets \hat{\mathcal{Z}}(p)$;\quad $T^* \gets count$;

    \ \ \ \ \ \ \ \ \textbf{end if}

    \ \ \ \ \textbf{else}

    \ \ \ \ \ \ \ \ $count \gets count - 1$; \tcp*{If the current stabber moves to a right endpoint, one less interval is stabbed.}

    \ \ \ \ \textbf{end if}

    \textbf{end for}
    \caption{Interval Stabbing (\cref{theorem:inter_stab}).}
    \label{algo:inter_stab}
\end{algorithm}

\subsection{Upper Bound}
\label{subsec:UB_CM}

Given a branch $[\underline{\br_1}, \overline{\br_1}]$ of $\br_1$, \eg, $[\alpha_1,\alpha_2]\times[\beta_1,\beta_1]$, we define $b_i := y_{i1} - \br_1^\top \bx_i$ as \cref{subsec:tear_lb}. Based on \cref{prop:LB-range}, the range $[b_{il}, b_{iu}]$ of $b_i$ can be easily computed. 
Moreover, we have the following proposition:
\begin{proposition}
    We can compute an upper bound $\overline{U}_{CM}$ of \textcolor{cvprblue}{CM-1} via solving \begin{equation}
        \overline{U}_{CM} = \max_{t_1\in \bbR} \sum_{i=1}^N \bm{1} \left( t_j \in [b_{il} - \xi_{i1}, b_{iu} + \xi_{i1}] \right).
    \end{equation}
    Specifically, $\overline{U}_{CM}$ can be computed in $O(N\log N)$ time by  \cref{algo:inter_stab}. 
\end{proposition}
\begin{proof}
We first consider \textcolor{cvprblue}{CM-1} constrained on the branch $[\underline{\br_1},\ \overline{\br_1}]$, that is,
\begin{equation}
    \max_{t_1\in \bbR} \max_{\br_1\in[\underline{\br_1},\ \overline{\br_1}]} \sum_{i=1}^N \bm{1} \left( | y_{i1} - \br_1^\top \bx_i - t_1| \leq \xi_{i1}  \right). \label{eq:CM_UB_1}
\end{equation}
To compute an upper bound of \eqref{eq:CM_UB_1}, we can relax it into
\begin{align}
    \max_{t_1\in \bbR} \sum_{i=1}^N \max_{\br_1\in[\underline{\br_1},\ \overline{\br_1}]} \bm{1} \left( | y_{i1} - \br_1^\top \bx_i - t_1| \leq \xi_{i1}  \right). \label{eq:CM_UB_2}
\end{align}
With the notation $b_i := y_{i1} - {\br_1}^{\top}\bx_i$ and the bounds $b_{il}$ and $b_{iu}$ in \cref{prop:LB-range}, we have
\begin{subequations}
    \begin{align}
    & \max_{t_1\in \bbR} \sum_{i=1}^N \max_{\br_1\in[\underline{\br_1},\ \overline{\br_1}]} \bm{1} \left( | y_{i1} - \br_1^\top \bx_i - t_1| \leq \xi_{i1}  \right) \notag \\
    = & \max_{t_1\in \bbR} \sum_{i=1}^N \max_{b_i\in[b_{il}, b_{iu}]} \bm{1} \left( b_i - \xi_{i1} \leq t_1 \leq b_i + \xi_{i1} \right) \label{eq:CM_UB_3}\\
    \leq & \max_{t_1\in \bbR} \sum_{i=1}^N \bm{1} \left( b_{il} - \xi_{i1} \leq t_1 \leq b_{iu} + \xi_{i1} \right) \label{eq:CM_UB_4} \\
    = & \max_{t_1\in \bbR} \sum_{i=1}^N \bm{1} \left( t_1 \in [b_{il} - \xi_{i1},\ b_{iu} + \xi_{i1}] \right) = \overline{U}_{CM}. \label{eq:CM_UB_5}
    \end{align}
\end{subequations}
Thus $\overline{U}_{CM}$ is indeed an upper bound of \textcolor{cvprblue}{CM-1}. Note that \eqref{eq:CM_UB_5} is exactly an interval stabbing problem as  \eqref{eq:interval_stab}, therefore $\overline{U}_{CM}$ can be solved in $O(N\log N)$ time by \cref{algo:inter_stab}.
\end{proof}

\subsection{Lower Bound}
\label{subsec:LB_CM}

We choose the center $\dot{\br}_1$ of the given branch $[\underline{\br_1}, \overline{\br_1}]$ to compute a lower bound $\underline{L}_{CM}$ of \textcolor{cvprblue}{CM-1}, \ie,
\begin{equation}\label{eq:CM_LB}
    \begin{split}
        \underline{L}_{CM} = \max_{t_1\in \bbR } \sum_{i=1}^N \bm{1} \left( | y_{i1} - \dot{\br}_1^\top \bx_i - t_1| \leq \xi_{i1}  \right).
    \end{split}
\end{equation}
Define $a_i:=y_{i1} - \dot{\br}_1^{\top}\bx_i$, $l_{ai}:=a_1 - \xi_{i1}$, and $u_{ai}:=a_1 + \xi_{i1}$. As in \cref{subsec:UB_CM}, let us rewrite \eqref{eq:CM_LB} as follows:
\begin{subequations}
    \begin{align}
        \underline{L}_{CM} &= \max_{t_1\in \bbR } \sum_{i=1}^N \bm{1} \left( a_i - \xi_{i1} \leq t_1 \leq  a_i + \xi_{i1} \right) \\
        &= \max_{t_1\in \bbR } \sum_{i=1}^N \bm{1} \left( y_{i1} - l_{ai} \leq t_1 \leq u_{ai} \right) \\
        &= \max_{t_1\in \bbR } \sum_{i=1}^N \bm{1} \left( t_i\in [l_{ai},\ u_{ai}]\right). \label{eq:CM_LB_final}
    \end{align}
\end{subequations}
We have transformed \eqref{eq:CM_LB} into an interval stabbing problem \eqref{eq:CM_LB_final}. Therefore we can compute the lower bound $\underline{L}_{CM}$ in $O(N\log N)$ time by \cref{algo:inter_stab}.

\section{Bounds for \textcolor{cvprblue}{TLS-1}}
\label{sec:bounds_TLS}


\begin{figure}[!t]
    \centering
    \includegraphics[width=0.75\textwidth]{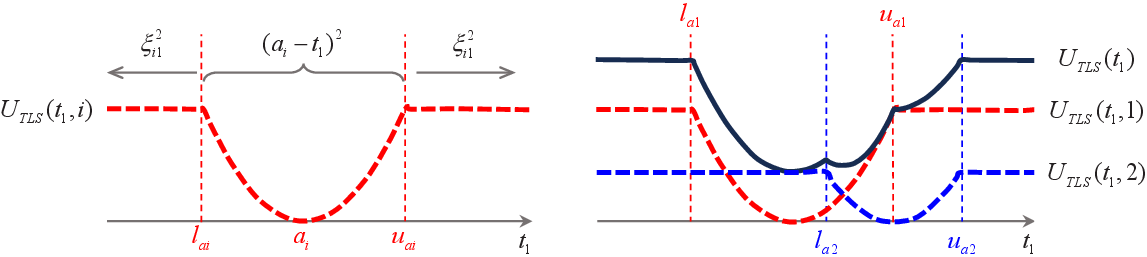}
    \\
    \makebox[0.3\textwidth]{\footnotesize (a) \quad \quad \ \ }
    \makebox[0.3\textwidth]{\footnotesize \quad \quad (b)}
    \\[-0.5em]
    \caption{(a) Illustration of $U_{TLS}(t_1,i) = \min\{(a_i - t_1)^2,\ \xi^2_{i1}\}$; (b) Illustration of $U_{TLS}(t_1)=\sum_{i=1}^{N} U_{TLS}(t_1,i)$ when $N=2$. (\cf \cref{subsec:UB_TLS}).}
    \label{fig:UB_illus_TLS}
\end{figure}

\subsection{Upper Bound}
\label{subsec:UB_TLS}

Similarly to \cref{subsec:tear_ub}, we choose the center $\dot{\br}_1$ of the given branch $ [\underline{\br_1}, \overline{\br_1}]$ to compute an upper bound $\overline{U}_{TLS}$ of this branch. Let $a_i:=y_{i1} - \dot{\br}_1^{\top}\bx_i$ as \cref{subsec:tear_ub} and \cref{subsec:LB_CM}, we have 
\begin{equation}
    \begin{split}\label{eq:UB_t1_TLS}
        \overline{U}_{TLS} &= \min_{t_1\in\mathbb{R}} \sum_{i=1}^{N} \min\{(a_i - t_1)^2,\ \xi^2_{i1}\} \\
        &= \min_{t_1\in\mathbb{R}} U_{TLS}(t_1),
    \end{split}
\end{equation}
where $U_{TLS}(t_1)$ is defined to be the sum of $U_{TLS}(t_1,i):=\min\{(a_i - t_1)^2,\ \xi^2_{i1}\}$, that is $U_{TLS}(t_1) := \sum_{i=1}^{N} U_{TLS}(t_1,i))$. We begin by the following auxiliary result (easy to prove):
\begin{theorem}\label{theorem:solve_qua}
    Consider the following problem:
    \begin{equation}
        \begin{split}\label{eq:LS_branch}
            \min_{d\in[d_l,\ d_u]} \sum_{q = 1}^{Q} (\gamma_q - d)^2.
        \end{split}
    \end{equation}
    Let $\tilde{d}:=\frac{1}{Q}\sum_{q=1}^{Q} \gamma_q$, then the global optimizer $\hat{d}$ to \eqref{eq:LS_branch} is given as follows:
    \begin{equation}\label{eq:conds-qua}
        \begin{split}
            \hat{d} = \begin{cases}
                \tilde{d}, & \tilde{d} \in [d_l,\ d_u]; \\
                d_u, & \tilde{d} > d_u; \\
                d_l, & \tilde{d} < d_l. \\
            \end{cases}
        \end{split}
    \end{equation}
\end{theorem}

\begin{algorithm}[!t]
    \DontPrintSemicolon
    \textbf{Input:} $\{l_{ai}\}_{i=1}^{N}$, $\{u_{ai}\}_{i=1}^{N}$, and $\{a_{i}\}_{i=1}^{N}$;

    \textbf{Output:} A global minimizer ${t_1}'$ and the minimum value $\overline{U}_{TLS}$ of \eqref{eq:UB_t1_TLS};

    $\{\lambda_k\}_{k=1}^{2N} \gets $ Sort $\{l_{ai}\}_{i=1}^{N} \cup \{u_{ai}\}_{i=1}^{N}$;
        
        $U_0 \gets \sum_{i=1}^{N} \xi_{i1}^2$;\quad $\overline{U}_{TLS} \gets U_0$;  

        $U_0^q \gets 0$;\quad $M_0^q\gets0$;\quad $\overline{a_0}\gets0$; \quad $U_0^c \gets U_0$;
 
	\textbf{for} $k\gets 1$ \textbf{to} $2N - 1$ \textbf{do}:

        \ \ \ \ $\mathcal{I}_k \gets [\lambda_k,\ \lambda_{k+1}]$;
        
        \ \ \ \ \textbf{if} $\lambda_{k} \in \{l_{ai}\}_{i=1}^{N}$ \textbf{then}

        \ \ \ \ \ \ \ \  $M_k^q \gets M_{k-1}^q + 1$;  \tcp*{Since $\lambda_k\in \{l_{ai}\}_{i=1}^{N}$, we enter a \textit{quadratic} regime.}
         
        \ \ \ \ \ \ \ \ $k_{in} \gets$ Index of $\lambda_k$ in $\{l_{ai}\}_{i=1}^{N}$;

        \ \ \ \ \ \ \ \ $a_k \gets $ $k_{in}^{th}$ element in $\{a_i\}_{i=1}^{N}$;

        \ \ \ \ \ \ \ \ $\overline{a_k} \gets \frac{M_{k-1}^q}{M_{k}^q} \cdot \overline{a_{k-1}} + \frac{1}{M_{k}^q}\cdot a_k$;

        \ \ \ \ \ \ \ \ ${t_{1k}} \gets$ Compare $\overline{a_k}$ with $\mathcal{I}_k$ based on \eqref{eq:conds-qua};  \tcp*{Compute the minimizer $t_{1k}$ on the interval $\mathcal{I}_k$ (see \cref{theorem:solve_qua}).}

        \ \ \ \ \ \ \ \ $U_k^q \gets U_{k-1}^q + {a_k}^2 + M_k^q\cdot t_{1k} \cdot (t_{1k} - 2\overline{a_k}) - M_{k-1}^q\cdot t_{1(k-1)} \cdot (t_{1(k-1)} - 2\overline{a_{k-1}})$;

        \ \ \ \ \ \ \ \ $U_k^c \gets U_{k-1}^c - \xi_{k_{in}1}^2$;

        \ \ \ \ \textbf{else if} $\lambda_{k} \in \{u_{ai}\}_{i=1}^{N}$ \textbf{then}

        \ \ \ \ \ \ \ \  $M_k^q \gets M_{k-1}^q - 1$; \tcp*{Since $\lambda_k\{u_{ai}\}_{i=1}^{N}$, we leave a \textit{quadratic} regime.}
         
        \ \ \ \ \ \ \ \ $k_{out} \gets$ Index of $\lambda_k$ in $\{u_{ai}\}_{i=1}^{N}$;

        \ \ \ \ \ \ \ \ $a_k \gets $ $k_{out}^{th}$ element in $\{a_i\}_{i=1}^{N}$;

        \ \ \ \ \ \ \ \ $\overline{a_k} \gets \frac{M_{k-1}^q}{M_{k}^q} \cdot \overline{a_{k-1}} - \frac{1}{M_{k}^q}\cdot a_k$;

        \ \ \ \ \ \ \ \ ${t_{1k}} \gets$ Compare $\overline{a_k}$ with $\mathcal{I}_k$ based on \eqref{eq:conds-qua}; \tcp*{Compute the minimizer $t_{1k}$ on th interval $\mathcal{I}_k$ (similarly to Line 13).}

        \ \ \ \ \ \ \ \ $U_k^q \gets U_{k-1}^q - {a_k}^2 + M_k^q\cdot t_{1k} \cdot (t_{1k} - 2\overline{a_k}) - M_{k-1}^q\cdot t_{1(k-1)} \cdot (t_{1(k-1)} - 2\overline{a_{k-1}})$;

        \ \ \ \ \ \ \ \ $U_k^c \gets U_{k-1}^c + \xi_{k_{out}1}^2$;

        \ \ \ \ \textbf{end if}
        
        \ \  \ \ $U_{k} \gets U_{k}^q + U_{k}^c$;

        \ \ \ \ \textbf{if} $U_{k} < \overline{U}_{TLS}$ \textbf{then}  \tcp*{Update $\overline{U}_{TLS}$ and the global minimizer ${t_1}'$ once $U_{k}$ is smaller than $\overline{U}_{TLS}$.}

        \ \ \ \ \ \ \ \ ${t_1}' \gets {t_{1k}}$;\quad $\overline{U}_{TLS} \gets U_{k}$;

        \ \ \ \ \textbf{end if}
        
        \textbf{end for}
        
	\caption{Globally Optimal 1D TLS Solver for Upper Bound Computation~\eqref{eq:UB_t1_TLS}.}
 \label{algo:UB-$t_1$_TLS}
\end{algorithm}

Let us now consider computing  $\overline{U}_{TLS} $ in \eqref{eq:UB_t1_TLS}. Define $l_{ai}:=a_i-\xi_{i1}$, $u_{ai}:=a_i+\xi_{i1}$.
Let $\{\lambda\}_{k=1}^{2N}$ be a sorted version of the $2N$ numbers $\{l_{ai}\}_{i=1}^{N} \cup \{u_{ai}\}_{i=1}^N$. On each interval $\mathcal{I}_k := [\lambda_k, \lambda_{k+1}]$, it is clear that $U_{TLS}(t_1,i)$ can be only the part of \textit{quadratic} $(a_i-t_1)^2$ or the \textit{constant} $\xi_{i1}^2$~(see \cref{fig:UB_illus_TLS}b). 

Based on \cref{theorem:solve_qua} and the above notation, we develop \cref{algo:UB-$t_1$_TLS} to get the globally optimal solution of~\eqref{eq:UB_t1_TLS} in $O(N\log N)$ time. \cref{algo:UB-$t_1$_TLS} iterates over $\{\lambda_k\}_{k=1}^{2N}$ and computes the minimal value $U_k$ of $U_{TLS}(t_1)$ on each interval $\mathcal{I}_k$. 
Obviously the minimum $U_k$ in $\{U_k\}_{k=1}^{2N-1}$ is exactly the minimum value $\overline{U}_{TLS}$, and next we will show that it can be computed incrementally during the iteration. Define $U_k^q$ to be the minimal value of the sum of $U_{TLS}(t_1,i)$s being \textit{quadratic} on each interval $\mathcal{I}_k$, as well as $U_k^c$ to be sum of $U_{TLS}(t_1,i)$s being \textit{constant}. It is clear that $U_k = U_k^q + U_k^c$. Based on \cref{theorem:solve_qua}, we have
\begin{equation}
    \begin{split}
        U_{k}^q &= \sum_{m \in \mathcal{M}_{k}^q} (a_m - {t_{1k}})^2 \\
        &= \sum_{m \in \mathcal{M}_{k}^q}{a_m}^2 + M_k^q\cdot {t_{1k}}^2 - 2{t_{1k}}\cdot{M_k^q}\cdot\overline{a_k},
    \end{split}
\end{equation}
where $\mathcal{M}_k^q$ is defined to be the set of indices of $U_{TLS}(t_1,i)$s being \textit{quadratic} on the interval $\mathcal{I}_k$, ${t_{1k}}$ is the corresponding value of $t_1$ leading to $U_k^q$ in $\mathcal{I}_k$~(that is, the $\hat{d}$ given by \cref{theorem:solve_qua}), $M_k^q$ is the set size of $\mathcal{M}_k^q$, and $\overline{a_k}$ is the average value of related $a_i$s in $\mathcal{M}_k^q$. At each iteration, the index set $\mathcal{M}_k^q$ differs from $\mathcal{M}_{k-1}^q$ by at most one element. More precisely:
\begin{itemize}
    \item If $\lambda_k\in \{l_{ai}\}_{i=1}^N$, then $\mathcal{M}_k^q$ has one more index $i$ than $\mathcal{M}_{k-1}^q$, and for this index $i$, the component $U_{TLS}(t_1,i)$ is quadratic on interval $\mathcal{I}_k$. In this case, we say \textit{we enter a quadratic regime}.
    \item If $\lambda_k\in \{u_{ai}\}_{i=1}^N$, then $\mathcal{M}_k^q$ has one less index $i$ than $\mathcal{M}_{k-1}^q$, and for this index $i$, the component $U_{TLS}(t_1,i)$ is constant (or no longer quadratic) on  interval $\mathcal{I}_k$. In this case, we say \textit{we leave a quadratic regime}.
\end{itemize}
In either case, we can update $U_k^q$ from $U_{k-1}^q$ as follows:
\begin{subequations}
    \begin{align}
        \overline{a_k} &= \begin{cases}
            \frac{M_{k-1}^q}{M_{k-1}^q + 1} \cdot \overline{a_{k-1}} + \frac{1}{M_{k-1}^q + 1}\cdot a_k, &\lambda_k \in \{l_{ai}\}_{i=1}^N;\\
            \frac{M_{k-1}^q}{M_{k-1}^q - 1} \cdot \overline{a_{k-1}} - \frac{1}{M_{k-1}^q-1}\cdot a_k, &\lambda_k \in \{u_{ai}\}_{i=1}^N;
        \end{cases}\\
        U_k^q - U_{k-1}^q &= \begin{cases}
             {a_k}^2 + M_k^q\cdot t_{1k} \cdot (t_{1k} - 2\overline{a_k}) - M_{k-1}^q\cdot t_{1(k-1)} \cdot (t_{1(k-1)} - 2\overline{a_{k-1}}), &\lambda_k \in \{l_{ai}\}_{i=1}^N;\\
             -{a_k}^2 + M_k^q\cdot t_{1k} \cdot (t_{1k} - 2\overline{a_k}) - M_{k-1}^q\cdot t_{1(k-1)} \cdot (t_{1(k-1)} - 2\overline{a_{k-1}}), &\lambda_k \in \{u_{ai}\}_{i=1}^N.\\
        \end{cases} \label{eq:TLS_ub_ite_q}
    \end{align}
\end{subequations}
Based on \eqref{eq:TLS_ub_ite_q}, we can easily compute $U_k^q$ in constant time~(Lines 10-14 and 18-22) at each iteration. As to $U_k^c$, at each iteration it also only differ with $U_{k-1}^c$ from one element, that is, containing one less \textit{constant} term when $\lambda_k$ belongs to $\{l_{ai}\}_{i=1}^N$, or containing one more \textit{constant} term when $\lambda_k$ belongs to $\{u_{ai}\}_{i=1}^N$. Therefore $U_k^c$ can be also incrementally computed by adding or subtracting the threshold term at each iteration~(Lines 15 and 23). Then $U_k$ can be computed by summing $U_k^q$ and $U_k^c$ at each iteration.
Since the sorting operation leads to $O(N\log N)$ time complexity and the iteration leads to $O(N)$ time complexity, the $\overline{U}_{TLS}$ can be computed in $O(N\log N)$ time by the proposed~\cref{algo:UB-$t_1$_TLS}.
In addition, based on the comparison at each iteration~(Lines 26-27), \cref{algo:UB-$t_1$_TLS} can solve \eqref{eq:UB_t1_TLS} to global optimality.

\subsection{Lower Bound}
\label{subsec:LB_TLS}

Let $b_i:=y_{i1} - {\br_1}^{\top}\bx_i$ as \cref{subsec:tear_lb} and \cref{subsec:UB_CM}, then the range $[b_{il},b_{iu}]$ of $b_i$ in the given branch can be solved by using~\cref{prop:LB-range}. To compute a lower bound, it suffices to relax \textcolor{cvprblue}{TLS-1} into the following problem:
\begin{equation}\label{eq:LB_bi_t1_TLS}
    \begin{split}
        \underline{L}_{TLS} &= \min_{t_1\in\mathbb{R},\ b_i\in[b_{il}, b_{iu}]} \sum_{i=1}^{N} \min\{( b_i - t_1 )^2,\ \xi_{i1}^2\} \\
    &= \min_{t_1\in\mathbb{R},\ b_i\in[b_{il}, b_{iu}]} \sum_{i=1}^{N} L_{TLS}(t_1,b_i,i),
    \end{split}
\end{equation}
where $L_{TLS}(t_1,b_i,i):= \min\{( b_i - t_1 )^2,\ \xi_{i1}^2\}$. Note that $\forall t_1 \in [b_{il}, b_{iu}]$, $b_i$ could be set by $t_1$ so as to minimize $L_{TLS}(t_1,b_i,i)$, accordingly we define
\begin{equation}
    \begin{split}
        L_{TLS}(t_1,i) &= \min_{b_i\in[b_{il}, b_{iu}]} L_{TLS}(t_1,b_i,i) \\ 
    &= \begin{cases}
        ( b_{il} - t_1 )^2, \quad&t_i\in[b_{il} - \xi_{i1},\ b_{il}]; \\
        0,  \quad&t_i\in[b_{il},\ b_{iu}]; \\
        ( b_{iu} - t_1 )^2,  \quad&\makecell{t_i\in [b_{iu},\ b_{iu} + \xi_{i1}]}; \\
        \xi_{i1}^2,  \quad&\textnormal{otherwise}.
    \end{cases}
    \end{split}
\end{equation}
Then \eqref{eq:LB_bi_t1_TLS} can be transformed equivalently to:
\begin{equation}\label{eq:LB_t1_TLS}
    \begin{split}
        \underline{L}_{TLS} &= \min_{t_1\in\mathbb{R}} \sum_{i=1}^{N} \min_{b_i\in[b_{il}, b_{iu}]} L_{TLS}(t_1,b_i,i) \\
        & = \min_{t_1\in\mathbb{R}} \sum_{i=1}^{N} L_{TLS}(t_1,i) \\
    &= \min_{t_1\in\mathbb{R}} L_{TLS}(t_1),
    \end{split}
\end{equation}
where $L_{TLS}(t_1)$ is defined to be the sum of $L_{TLS}(t_1,i)$, that is $L_{TLS}(t_1):=\sum_{i=1}^{N}L_{TLS}(t_1,i)$. Define $l_{bi} := b_{il} - \xi_{i1}$, $u_{bi} := b_{iu} + \xi_{i1}$, and $L_{TLS}(t_1):=\sum_{i=1}^{N}L_{TLS}(t_1,i)$. Let $\{\psi_k\}_{k=1}^{4N}$ be a sorted version of $\{b_{il}\}_{i=1}^{N} \cup \{b_{iu}\}_{i=1}^{N} \cup \{l_{bi}\}_{i=1}^{N} \cup \{u_{bi}\}_{i=1}^{N}$. Then on each interval $\mathcal{J}_k:=[\psi_k, \psi_{k+1}]$, it is obvious that $L_{TLS}(t_1,i)$ can be only one the \textit{two quadratics} $\{( b_{il} - t_1 )^2,\ ( b_{iu} - t_1 )^2\}$ or one of the \textit{two constants} $\{0,\ \xi_{i1}^2\}$~(see \cref{fig:LB_illus_TLS}). In the following, we develop \cref{algo:LB-$t_1$_TLS} to get the globally optimal solution of \eqref{eq:LB_t1_TLS} with $O(N^2)$ time complexity based on the above notions.

\begin{figure}[!t]
    \centering
    \includegraphics[width=0.75\textwidth]{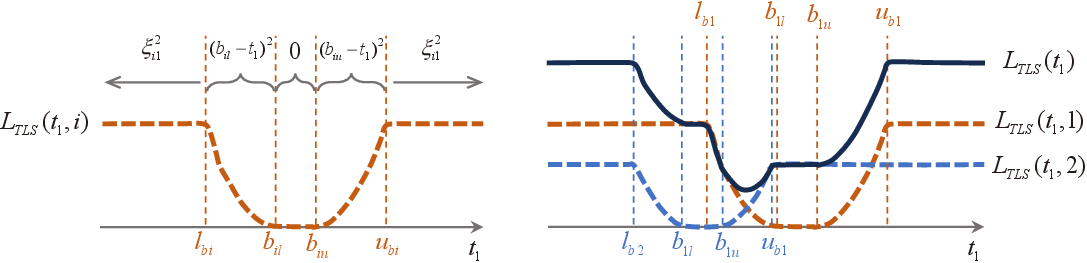}
    \\
    \makebox[0.3\textwidth]{\footnotesize \ \ (a)}
    \makebox[0.3\textwidth]{\footnotesize \quad \quad (b)}
    \\[-0.5em]
    \caption{(a) Illustration of $L_{TLS}(t_1,i) = \min_{b_i\in[b_{il}, b_{iu}]} L_{TLS}(t_1,b_i,i)$, where $L_{TLS}(t_1,b_i,i) = \min\{( b_i - t_1 )^2,\ \xi_{i1}^2\}$; (b) Illustration of $L_{TLS}(t_1)=\sum_{i=1}^{N}L_{TLS}(t_1,i)$ when $N=2$. (\cf \cref{subsec:LB_TLS}).}
    \label{fig:LB_illus_TLS}
\end{figure}

\begin{algorithm}[!t]
    \DontPrintSemicolon
    \textbf{Input:} $\{b_{il}\}_{i=1}^{N}$, $\{b_{iu}\}_{i=1}^{N}$,$\{l_{bi}\}_{i=1}^{N}$, and $\{u_{bi}\}_{i=1}^{N}$;

    \textbf{Output:} A global minimizer $\hat{t}_1$ and the minimum value $\underline{L}_{TLS}$ of \eqref{eq:LB_t1_TLS};

    $\{\psi_k\}_{k=1}^{4N} \gets $ Sort $\{b_{il}\}_{i=1}^{N} \cup \{b_{iu}\}_{i=1}^{N} \cup \{l_{bi}\}_{i=1}^{N} \cup$ $\{u_{bi}\}_{i=1}^{N}$;
        
        $L_0 \gets \sum_{i=1}^{N} \xi_{i1}^2$;\quad $\underline{L}_{TLS} \gets L_0$; 

        $\mathcal{N}_0^q \gets \emptyset$;\quad $L_0^c \gets L_0$;
 
	\textbf{for} $k\gets 1$ \textbf{to} $4N - 1$ \textbf{do}:

         \ \ \ \ $\mathcal{J}_k = [\lambda_k,\ \lambda_{k+1}]$;
        
        \ \ \ \ \textbf{if} $\psi_{k} \in \{l_{bi}\}_{i=1}^{N}$ \textbf{then}

        \ \ \ \ \ \ \ \ $k' \gets$ Index of $\psi_k$ in $\{l_{bi}\}_{i=1}^{N}$;

        \ \ \ \ \ \ \ \ $\mathcal{N}_k^q \gets \mathcal{N}_{k-1}^q \cup \{k'\}$; \tcp*{Since $\psi_k \in \{l_{bi}\}_{i=1}^{N}$, we enter a \textit{quadratic regime}.}

        \ \ \ \ \ \ \ \ $L_k^c \gets L_{k-1}^c - \xi_{k_{in}1}^2$;

        \ \ \ \ \textbf{else if} $\psi_{k} \in \{b_{il}\}_{i=1}^{N}$ \textbf{then}

        \ \ \ \ \ \ \ \ $k' \gets$ Index of $\psi_k$ in $\{b_{il}\}_{i=1}^{N}$;

        \ \ \ \ \ \ \ \ $\mathcal{N}_k^q \gets \mathcal{N}_{k-1}^q \setminus \{k'\}$; \tcp*{Since $\psi_k \in \{b_{il}\}_{i=1}^{N}$, we leave a \textit{quadratic regime}.}

        \ \ \ \ \ \ \ \ $L_k^c \gets L_{k-1}^c$;

        \ \ \ \ \textbf{else if} $\psi_{k} \in \{b_{iu}\}_{i=1}^{N}$ \textbf{then}

        \ \ \ \ \ \ \ \ $k' \gets$ Index of $\psi_k$ in $\{b_{iu}\}_{i=1}^{N}$;

        \ \ \ \ \ \ \ \ $\mathcal{N}_k^q \gets \mathcal{N}_{k-1}^q \cup \{k'\}$; \tcp*{Since $\psi_k \in \{b_{iu}\}_{i=1}^{N}$, we enter a \textit{quadratic regime}.}

        \ \ \ \ \ \ \ \ $L_k^c \gets L_{k-1}^c$;

        \ \ \ \ \textbf{else if} $\psi_{k} \in \{u_{bi}\}_{i=1}^{N}$ \textbf{then}

         \ \ \ \ \ \ \ \ $k' \gets$ Index of $\psi_k$ in $\{u_{bi}\}_{i=1}^{N}$;

        \ \ \ \ \ \ \ \ $\mathcal{N}_k^q \gets \mathcal{N}_{k-1}^q \setminus \{k'\}$; \tcp*{Since $\psi_k \in \{u_{bi}\}_{i=1}^{N}$, we leave a \textit{quadratic regime}.}

        \ \ \ \ \ \ \ \ $L_k^c \gets L_{k-1}^c + \xi_{k_{out}1}^2$;

        \ \ \ \ \textbf{end if}

        \ \ \ \  $\hat{t}_{1k},\ L_{k}^q \gets \min_{t_{1k}\in\mathcal{J}_k} \sum_{n\in\mathcal{N}^q} (a_n - t_{1k})^2$ based on \cref{theorem:solve_qua}; \tcp*{Compute the minimizer $\hat{t}_{1k}$ and minimal value $L_{k}^q$ of the sum of \textit{quadratic} terms on the interval $\mathcal{J}_k$ (see \cref{theorem:solve_qua}).}
        
        \ \  \ \ $L_{k} \gets L_{k}^q + L_{k}^c$;

        \ \ \ \ \textbf{if} $L_{k} < \underline{L}_{TLS}$ \textbf{then} \tcp*{Update $\underline{L}_{TLS}$ and the global minimizer $\hat{t}_1$ once $L_{k}$ is smaller than $\underline{L}_{TLS}$.}

        \ \ \ \ \ \ \ \ $\hat{t}_1 \gets \hat{t}_{1k}$;\quad $\underline{L}_{TLS} \gets L_{k}$;

        \ \ \ \ \textbf{end if}
        
        \textbf{end for}
        
	\caption{Globally Optimal 1D TLS Solver for Lower Bound Computation~\eqref{eq:LB_t1_TLS}.}
 \label{algo:LB-$t_1$_TLS}
\end{algorithm}

Specifically, we iterate over $\{\psi_k\}_{k=1}^{4N}$ to compute the minimal value $L_k$ of $L_{TLS}(t_1)$ on each interval $\mathcal{J}_k$, and the minimum one in $\{L_k\}_{k=1}^{4N-1}$ is exactly the $\underline{L}_{TLS}$. 
On each $\mathcal{J}_k$, we define a set $\mathcal{N}_k^q$ to store the indices of $L_{TLS}(t_1,i)$s being \textit{quadratic}, as well as a variable $L_k^c$ to store the sum of $L_{TLS}(t_1,i)$s being \textit{constant}.
At each iteration, the index set $\mathcal{N}_k^q$ differ from $\mathcal{N}_{k-1}^q$ by at most one element and the variable $L_k^c$ can be updated from $L_{k-1}^c$.
Suppose $k'$ is the index of $\psi_k$ in one of the sets $\{l_{bi}\}_{i=1}^{N}$/$\{b_{il}\}_{i=1}^{N}$/$\{b_{iu}\}_{i=1}^{N}$/$\{u_{bi}\}_{i=1}^{N}$, we have
\begin{itemize}
    \item If $\psi_k \in \{l_{bi}\}_{i=1}^{N}$, then $\mathcal{N}_k^q$ has one more index $k'$ than $\mathcal{N}_{k-1}^q$, and $L_k^c$ has one less constant term $\xi_{k'}^2$ than $L_{k'}^c$. In this case, similarly to \cref{subsec:UB_TLS}, we say \textit{we enter a quadratic regime}.
    \item If $\psi_k \in \{b_{il}\}_{i=1}^{N}$, then $\mathcal{N}_k^q$ has one less index $k'$ than $\mathcal{N}_{k-1}^q$, and $L_k^c$ equals to $L_{k-1}^c$. In this case, similarly to \cref{subsec:UB_TLS}, we say \textit{we leave a quadratic regime}.
    \item If $\psi_k \in \{b_{iu}\}_{i=1}^{N}$, then $\mathcal{N}_k^q$ has one more index $k'$ than $\mathcal{N}_{k-1}^q$, and $L_k^c$ equals to $L_{k-1}^c$. In this case, \textit{we enter a quadratic regime}.
    \item If $\psi_k \in \{u_{bi}\}_{i=1}^{N}$, then $\mathcal{N}_k^q$ has one less index $k'$ than $\mathcal{N}_{k-1}^q$, and $L_k^c$ has one more constant term $\xi_{k'}^2$ than $L_{k'}^c$. In this case, \textit{we leave a quadratic regime}.
\end{itemize}
In either cases, $\mathcal{N}_k^q$ and $L_k^c$ can be updated. Then we can compute the minimal value $L_k^q$ of the sum of $L_{TLS}(t_1,i)$s that $i\in\mathcal{N}_k^q$~(\textit{quadratic} terms) based on \cref{theorem:solve_qua} with at most $O(N)$ time complexity, and furthermore the minimal value $L_k = L_k^q + L_k^c$ on each interval $\mathcal{J}_k$. 
Based on the comparison at each iteration~(Lines 27-28), \cref{algo:LB-$t_1$_TLS} can solve \eqref{eq:LB_t1_TLS} to global optimality with $O(N^2)$ time complexity.

\section{Proof of \cref{prop:LB-range}}
\label{appen: proof_prop_LB}
Let $\bx_i := [x_{i1},\ x_{i2},\ x_{i3}]^{\top}$, we have
    \begin{equation}\label{eq:bound_r_1}
        \begin{split}
            b_i &= y_{i1} - {\br_1}^{\top}\bx_i\\
            &=y_{i1} - x_{i1}\sin{\beta}\cos{\alpha} - x_{i2}\sin{\beta}\sin{\alpha} - x_{i3}\cos{\beta} \\
            &= y_{i1} - (x_{i1}\cos{\alpha}+x_{i2}\sin{\alpha})\sin{\beta} - x_{i3}\cos{\beta}\\
            &= y_{i1} - \sqrt{{x_{i1}}^2+{x_{i2}}^2}\cos{(\alpha - \alpha^{*})}\sin{\beta} - x_{i3}\cos{\beta},
        \end{split}
    \end{equation}
    where $\alpha^{*} \in [0, \pi]$ denotes the arc-tangent angle of $x_{i2}/x_{i1}$. Now consider the following lemma:
    \begin{lemma}\label{lem:bound_cos}
    Given $\theta \in [\theta_1, \theta_2] \subseteq [0, \pi]$ and $\phi \in [0, \pi]$, define $f(\theta):=cos(\theta-\phi)$, there is
    \begin{equation}\label{eq:bound_cos_1}
        \begin{split}
            f(\theta) \in \begin{cases}
                [f(\theta_1),\ f(\theta_2)], & \textnormal{if} \ \phi \geq \theta_2; \\
                [f(\theta_2),\ f(\theta_1)], & \textnormal{if} \ \phi \leq \theta_1; \\
                [\min\{f(\theta_1),\ f(\theta_2)\},\ 1], & \textnormal{otherwise}.
            \end{cases}
        \end{split}
    \end{equation}
    Else if $[\theta_1, \theta_2] \subseteq (\pi, 2\pi]$, there is
    \begin{equation}\label{eq:bound_cos_2}
        \begin{split}
            f(\theta) \in \begin{cases}
                [f(\theta_2),\ f(\theta_1)], & \textnormal{if} \ \phi \geq \theta_2; \\
                [f(\theta_1),\ f(\theta_2)], & \textnormal{if} \ \phi \leq \theta_1; \\
                [-1, \max\{f(\theta_1),\ f(\theta_2)\}], & \textnormal{otherwise.}
            \end{cases}
        \end{split}
    \end{equation}
    \end{lemma}
    Based on \cref{lem:bound_cos}~(easy to prove), the range of $\cos{(\alpha - \alpha^{*})}$ in \eqref{eq:bound_r_1} can be computed and we define it by $[\Psi_l,\ \Psi_u]$. Since $\sin{\beta} \geq 0$, we have
    \begin{equation}\label{eq:bound_r_2}
        \begin{split}
            y_{i1} - &\sqrt{{x_{i1}}^2 + {x_{i2}}^2}\Psi_u\sin{\beta} - x_{i3}\cos{\beta} \leq b_i \leq y_{i1} - \sqrt{{x_{i1}}^2 + {x_{i2}}^2}\Psi_l\sin{\beta} - x_{i3}\cos{\beta} \\
            \Leftrightarrow y_{i1} - & \sqrt{({x_{i1}}^2 + {x_{i2}}^2){\Psi_u}^2 + {x_{i3}}^2}\cos{(\beta - \beta_l^{*})} \leq b_i \leq y_{i1} - \sqrt{({x_{i1}}^2 + {x_{i2}}^2){\Psi_l}^2 + {x_{i3}}^2}\cos{(\beta - \beta_u^{*})},
        \end{split}
    \end{equation}
where $\beta_l^{*},\ \beta_u^{*} \in [0, \pi]$ denote the arc-tangent angles of $\big(\sqrt{{x_{i1}}^2 + {x_{i2}}^2}\Psi_u\big) / x_{i3}$ and $\big(\sqrt{{x_{i1}}^2 + {x_{i2}}^2}\Psi_l\big) / x_{i3}$, respectively. Note that the ranges of $\cos{(\beta - \beta_l^{*})}$ and $\cos{(\beta - \beta_u^{*})}$ can be solved based on \cref{lem:bound_cos}. Denote $\gamma_l^{u}$ as the upper bound of $\cos{(\beta - \beta_l^{*})}$, $\gamma_u^{l}$ as the lower bound of $\cos{(\beta - \beta_u^{*})}$, we have
\begin{equation}\label{eq:bound_r_3}
    \begin{split}
        b_{il} = y_{i1} - \gamma_l^{u}\sqrt{({x_{i1}}^2 + {x_{i2}}^2){\Psi_u}^2 + {x_{i3}}^2} \leq b_i \leq y_{i1} - \gamma_u^{l}\sqrt{({x_{i1}}^2 + {x_{i2}}^2){\Psi_l}^2 + {x_{i3}}^2} = b_{iu}.
    \end{split}
\end{equation}

Since $\Psi_l$, $\Psi_u$, $\gamma_l^{u}$, and $\gamma_u^{l}$ in \eqref{eq:bound_r_3} can be easily computed based on \cref{lem:bound_cos}, we can therefore get the range $[b_{il}, b_{iu}]$ of $b_i$ in \eqref{eq:bound_r_1} in constant time.

\section{Extra Experimental Details}
\myparagraph{Hyperparameter Setup} In Sec. 2.2 of the main manuscript, we decomposed the original 6-dimensional problem, \textcolor{red}{TEAR}, into two subproblems, \textcolor{red}{TEAR-1} and \textcolor{red}{TEAR-2}. \textcolor{red}{TEAR} has a threshold hyperparameter $\xi_i$, as is typical in many prior works. Moreover, \textcolor{red}{TEAR-1} has its own threshold $\xi_{i1}$ and \textcolor{red}{TEAR-2} has $\xi_{i2}$. 

However, this is not to say our method requires more hyperparameters than prior works. In fact, given the commonly used parameter $\xi_i$, we can choose $\xi_{i1}$ and $\xi_{i2}$ relatively easily, and here is how we do it. First, we simply set $\xi_{i1}$ to be equal to $\xi_{i}$. Second, recall the optimal solution $(\hat{\br}, \hat{t}_1)$ and the associated inlier indices $\hat{\cI}_1$ defined in (1) in the main manuscript. For each $i\in\hat{\cI}_1$, we set $\xi_{i2} $ to $ \xi_i - | y_{i1} - \hat{\br}_1^\top \bx_i - \hat{t}_1|$. Our experiments justify the choices of the hyperparameters. In addition, we set the minimal branch resolution in the BnB part of our method as 1e-3. As to TR-DE~\cite{Chen-CVPR2022c}, we set the resolution as 5e-2 to guarantee its experimental time limited in five days~(otherwise it costs averagely more than 200s for each real-world pair).

\label{sec:extra_exper}
\myparagraph{Dataset Details} 
For the three real-world datasets (3DMatch \cite{Zeng-CVPR2017}, KITTI \cite{Geiger-IJRR2013}, and ETH \cite{Theiler-ISPRS2014}), we follow~\cite{Chen-CVPR2022c, Chen-CVPR2022b, Li-arXiv2023} to set the inlier threshold $\xi_i$ based on the downsampling voxel size. Specifically, $\xi_i$ is set to 10 cm for the 3DMatch Dataset~\cite{Zeng-CVPR2017}, 60 cm for the KITTI Dataset~\cite{Geiger-IJRR2013}, and 30 cm for the ETH Dataset~\cite{Theiler-ISPRS2014}, respectively. 

In the Stanford 3D scanning dataset \cite{Curless-CCGIT1996}, we used 5 objects. And in Tab. 4 of the main manuscript, we reported the number of points each object contains, namely, \textit{Armadillo} has $10^5$ points, \textit{Happy Buddha} has $5\times10^5$ points, \textit{Asian Dragon} has $10^6$ points, \textit{Thai Statue} has $4\times10^6$ points, and \textit{Lucy} has $10^7$ points. We emphasize that each of these objects has slightly more points, and for clarity we downsampled them a little bit. In fact, \textit{Armadillo} has approximately $1.7\times10^5$ points, \textit{Happy Buddha} has approximately $5.4\times10^5$ points, \textit{Asian Dragon}  has approximately $3.6\times 10^6$ points, \textit{Thai Statue} has approximately $4.9\times10^6$ points, and \textit{Lucy} has approximately $1.4\times 10^7$ points.

\begin{figure}[!t]
    \centering
    \includegraphics[width=0.475\textwidth]{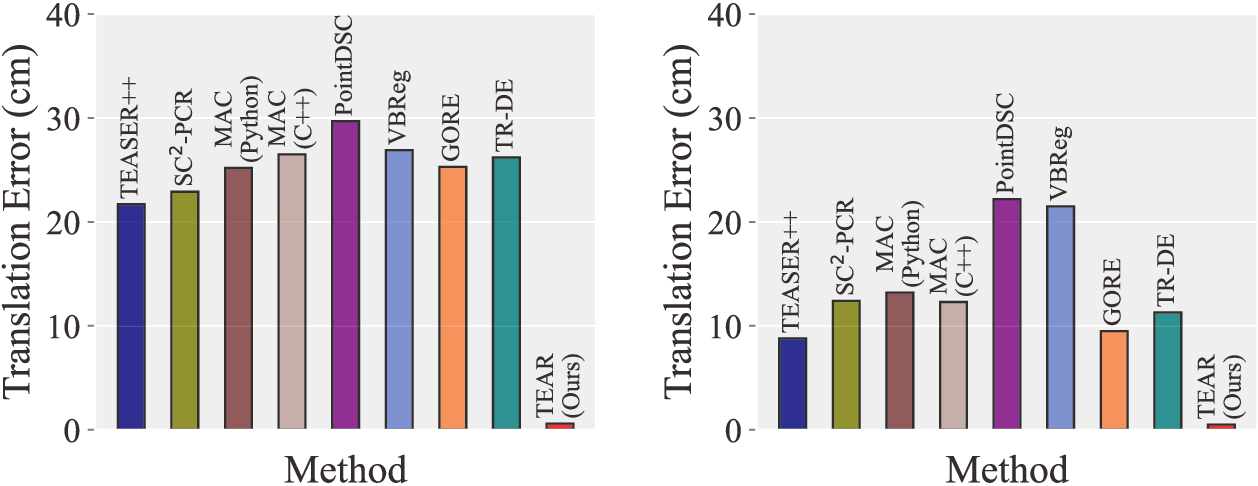}
    \\[-0.3em]
    \makebox[0.235\textwidth]{\footnotesize \quad \quad \ (a)}
    \makebox[0.235\textwidth]{\footnotesize \quad \quad \quad (b)}
    \\[-0.5em]
    \caption{Average translation errors of other methods in Tab. 4 of the manuscript taking as inputs the $10^4$ points downsampled from \textit{Lucy} that originally has $10^7$ point pairs (\cref{fig:huge_down_transla}a: 99.8\% outliers; \cref{fig:huge_down_transla}b: 95\% outliers). $\tear$ runs on the original $10^7$ input point pairs. 20 trials.}
    \label{fig:huge_down_transla}
\end{figure}

\myparagraph{Related Translation Errors of Fig. 7}
Recall that in Fig. 7 of the main manuscript, we report the average rotation errors of the unscalable methods in Tab. 4 of the manuscript evaluated on the downsampled data and our $\tear$ on the original data. As shown in \cref{fig:huge_down_transla}, we additionally report the related average translation errors.

\end{document}